\documentclass{article}
\usepackage{fullpage}
\usepackage{makecell}
\usepackage{multirow}
\usepackage[dvipsnames]{xcolor}

\usepackage{amsmath,amssymb,amsfonts,amsthm}
\usepackage{graphicx}
\usepackage{natbib}
\usepackage{algorithm}
\usepackage{algorithmic}
\usepackage{hyperref}
\usepackage{fullpage}
\usepackage{todonotes}
\usepackage{enumitem}

\usepackage[font=small, labelfont=bf]{caption}
\usepackage[font=small]{subcaption}

\newcommand{\R}{\mathbb{R}}
\newcommand{\E}{\mathbb{E}}
\newcommand{\mcD}{\mathcal{D}}
\newcommand{\mcN}{\mathcal{N}}
\newcommand{\mcL}{\mathcal{L}}

\newcommand{\mcF}{\mathcal{F}}
\newcommand{\mcG}{\mathcal{G}}
\newcommand{\mcX}{\mathcal{X}}
\newcommand{\mcY}{\mathcal{Y}}
\newcommand{\hw}{\hat{w}}
\newcommand{\dd}{\mathrm{d}}

\newcommand{\smax}{\operatorname{softmax}}

\DeclareMathOperator*{\argmax}{arg\,max}
\DeclareMathOperator*{\argmin}{arg\,min}
\newcommand{\var}{\operatorname{Var}}

\newcommand{\lv}{l_v}
\newcommand{\ws}{w^*}
\newcommand{\dx}{{d_x}}
\newcommand{\dw}{{d}}
\newcommand{\dhp}{{h}}

\newtheorem{assumption}{Assumption}
\newtheorem{app_assumption}{Assumption}[section]
\newtheorem{theorem}{Theorem}[section]
\newtheorem{lemma}[theorem]{Lemma}
\newtheorem{proposition}[theorem]{Proposition} 
\newtheorem{remark}[theorem]{Remark}

\newtheorem{definition}[theorem]{Definition}

\definecolor{commentblue}{RGB}{70, 130, 180}
\definecolor{commentgreen}{RGB}{34, 139, 34}
\definecolor{commentpurple}{RGB}{138, 43, 226}
\definecolor{commentteal}{RGB}{0, 128, 128}

\title{Generalization Guarantees on Data-Driven Tuning of Gradient Descent with Langevin Updates}
\author{
Saumya Goyal\thanks{Machine Learning Department, Carnegie Mellon University},
Rohith Rongali\thanks{Electrical and Computer Engineering Department, Carnegie Mellon University.\\
Correspondance to Saumya Goyal, $<$saumyag@andrew.cmu.edu$>$},
Ritabrata Ray\footnotemark[1],
Barnabás Póczos\footnotemark[1]
}
\date{}

\begin{document}

\maketitle

\begin{abstract}
    We study learning to learn through the lens of hyperparameter tuning.
    We propose the Langevin Gradient Descent Algorithm (LGD), which approximates the mean of the posterior distribution defined by the loss function and regularizer of a regression task with convex objective. For classification tasks, the LGD algorithm estimates the posterior probabilities of each class on the test set. We prove the existence of an optimal hyperparameter configuration for which the LGD algorithm achieves the Bayes' optimal solution for squared loss on regression tasks, and for which LGD closely approximates the posterior probabilities for well-specified classification tasks. Subsequently, we study generalization guarantees on meta learning optimal hyperparameters for the LGD algorithm from a given set of tasks in the data-driven setting. For a number of parameters $\dw$ and hyperparameter dimension $\dhp$, we show a pseudo-dimension bound of $O(\dw\dhp)$, up to logarithmic terms under mild assumptions on LGD. This matches the dependence of the bounds on number of parameters obtained in prior work for linear regression using the elastic net,
    which only allows for $\dhp=2$ hyperparameters, and extends their bounds to regression on convex loss. Compared to bounds on regularized logistic regression that allow for only $\dhp=1$ hyperparameter, our bounds improve greatly on the dependence on samples per task at the cost of worse dependence on the number of parameters by accounting for hardware-aware procedures. Finally, we show empirical evidence of the success of LGD and the meta learning procedure for few-shot learning on linear and logistic regression using synthetically created datasets.
\end{abstract}

\section{Introduction}
The success of modern machine learning models increasingly depends on the careful selection of hyperparameters, a task that remains one of the most challenging and resource intensive aspects of developing high-performance systems. Unlike model parameters, which are learned directly from data through optimization, hyperparameters must be specified prior to training and fundamentally shape the learning dynamics, model capacity, and generalization behavior. As models have grown in scale and complexity, from deep neural networks with billions of parameters to foundation models trained on vast datasets, the hyperparameter search space has expanded dramatically, making effective tuning both more critical and more computationally prohibitive. Traditional approaches to hyperparameter optimization have evolved from manual tuning and grid search \citep{JMLR:v13:bergstra12a} to more sophisticated methods like Bayesian optimization \citep{snoek2012practicalbayesianoptimizationmachine, NIPS2011_86e8f7ab} and multi-armed bandit algorithms \citep{li2018hyperbandnovelbanditbasedapproach}. 
While such strategies can discover well-performing hyperparameter combinations, they require extensive computational resources, often necessitating hundreds or thousands of complete training runs. 
Despite substantial progress in automated hyperparameter optimization, current methodologies exhibit a fundamental gap: they provide limited insight into the underlying mechanisms by which hyperparameters influence learning dynamics and fail to leverage the rich information available during training. Further, such heuristics often offer no guarantees on optimality, or require strong data-dependent assumptions \citep{enet_2022}. 

In this work, we approach hyperparameter tuning in a principled way by first investigating how hyperparameters affect learning dynamics. We propose the Langevin Gradient Descent Algorithm (LGD), which performs gradient descent with Langevin updates. We treat the regularizer and the learning rate as our hyperparameters, and show the existence of optimal hyperparameters that correspond to a prior distribution over tasks. Given optimal hyperparameters, we show that Langevin Gradient Descent samples from the posterior, allowing our estimator to achieve Bayes' optimality, something not guaranteed through a plan gradient descent. Thus, in our setting where tasks are sampled from an underlying prior distribution, learning to learn reduces to hyperparameter tuning of the hyperparameters of LGD. This contradicts with popular learning to learn approaches \citep{finn2017modelagnosticmetalearningfastadaptation,nichol2018firstordermetalearningalgorithms, franceschi2018bilevelprogramminghyperparameteroptimization, rajeswaran2019metalearningimplicitgradients, liu2021investigatingbileveloptimizationlearning}, where optimality of the learning algorithm is not guaranteed even with an optimal choice of the hyperparameters. 

We then study generalization guarantees on learning the optimal regularizer and learning rate from tasks for both regression and classification problems. For our empirical evaluation, we learn hyperparameters through gradient descent over the validation loss on observed tasks. This is similar to gradient-based hyperparameter tuning approaches of prior work \citep{hyperparam_grad_dougal15, hypergrad_baydin18, engstrom2025optimizingmltrainingmetagradient}.

Formally, we consider a regression problem with inputs $X\in \R^{n\times \dx}$ and outputs $y\in \R^n$, where $n$ is the number of training
samples and $\dx$ is the feature space dimension. For a function $f(x;w):\R^\dx\times \R^\dw \rightarrow \R$, 
and test/validation inputs $X_v\in \R^{n_v\times \dx}$, we want to find an estimator $\hat{f}_v(.,.,.):\R^{n_v\times \dx}\times\R^{n\times \dx}\times \R^n\rightarrow\R^{n_v} $ that takes as input $X_v,X,y$ and outputs an estimate of the validation/test output, $y_v\in\R^{n_v}$, 
that minimizes the validation error, $l(y_v,\hat{f}_v(X_v,X,y))$ for some loss function $l$. Similarly, we consider a classification problem with inputs $X\in \R^{n\times \dx}$ and outputs $y\in [m]^{n}$, where $m$ is the number of classes. We consider functions with signature $f(x;w):\R^\dx\times \R^\dw \rightarrow \R^m$, where we treat the outputs as ``weights" on different classes. For example, it is common in literature to treat the outputs as logits, or proportional to log-probabilities of classes. For test/validation inputs $X_v\in \R^{n_v\times \dx}$, we want to find an estimator $\hat{f}_v(.,.,.):\R^{n_v\times \dx}\times\R^{n\times \dx}\times \R^n\rightarrow\R^{n_v\times m} $ that takes as input $X_v,X,y$ and outputs an estimate of the true weights of classes for $y_v\in[m]^{n_v}$.
We want to find an estimator $\hat{f}_v$ that minimizes the validation error, $l(y_v,\hat{f}_v(X_v,X,y))$ for an appropriate loss function $l$. For mathematical convenience we will directly consider $\hat{F}_v \in \R^{n_v\times m}$, which are the predictions of probabilities of different classes of $y_v\in[m]^{n_v}$, instead of the weight vector predictions $\hat{f}_v$, such that the corresponding weight vector minimizes the validation loss.

Directly optimizing to minimize the validation error often leads to overfitting, which motivates the use of regularizers. Regularizers allow for the possibility of finding a good balance between fitting the training data and generalizing to unseen data. We incorporate regularizers in our objective such that the gradients of regularizers are of the form $r(w;\theta)\in\R^{\dw}$ where $\theta\in \Theta$ are hyperparameters. 
We restrict our study to the case when $l(y,f(X;w))$ is convex in $w$ and the regularizer is strongly convex in $w$ for all hyperparameters. Strong convexity of the objective is a common assumption in optimization that leads to tight bounds \citep{nesterov2018lectures}.

We propose a Langevin diffusion-inspired gradient descent on $w$ that follows: 
\begin{align*}
    w_{(k+1)} = w_{(k)} - \eta (\nabla l(y,f(X;w)) + r(w;\theta)) + \sqrt{2\eta}\xi_k,
\end{align*}
where $\xi_k$ is standard Gaussian noise. After the diffusion stage, we perform a consolidation step that 
averages the predictions of final few iterates of the diffusion stage. We begin with a brief background into statistical generalization guarantees and Langevin algorithms in Section \ref{sec:bckg}. 
We then formalize our setting in Section \ref{sec:setting}. We introduce the proposed Algorithms \ref{alg:lgd}, \ref{alg:lgd_class} in Section \ref{sec:lgd} and show that the proposed strategy converges close to the Bayes optimal solution for appropriately chosen $\theta$ and $\eta$ for specific commonly used loss functions. 

In Section \ref{sec:guarantees}, we discuss finding the optimal values of $\eta, \theta$. We assume access to a set of tasks sampled similarly, and propose using $\eta, \theta$ that minimize the average validation loss across tasks. This makes our setting similar to multi-task learning since we use previously seen tasks to develop a procedure for future, unseen tasks. We compute the pseudo-dimension of the function class corresponding to validation losses for $\eta,\theta$ in Theorem \ref{thm:pdim_comp} for regression and Theorem \ref{thm:pdim_class} for classification, and find a pseudo-dimension of $O(\dw\dhp)$, up to logarithmic terms for either case. For the case of linear regression with Elastic Net, this matches the worst-case bounds developed in \citet{balcan_23} up to logarithmic terms. While the bounds of \citet{balcan_23} degrade sharply with the number of training samples for logistic regression, we address this limitation by accounting for hardware approximations of complicated functions.
Interestingly, we show a dependence on the number of parameters $\dw$, instead of the feature space dimension $\dx$, and we extend their analysis to any regression problem with a convex objective. We show generalization guarantees on this procedure in Theorem \ref{thm:gen_guarantee}. We further show convergence of the hyperparameter tuning algorithm to the Bayes' optimal solution in Theorem \ref{thm:erm_bayes_conv} for regression and Theorem \ref{thm:erm_bayes_class} for classification. Finally, we present empirical evidence of the success of the proposed Langevin Gradient Descent algorithm and the meta learning procedure for few-shot learning on linear and logistic regression on a few synthetic datasets in Section \ref{sec:expts}, where we observe that the meta-learned algorithm performs on par with the optimal algorithm with as little as $1/10$th the number of iterations.

\subsection{Summary of Contributions and Technical Novelty}
Our contributions are summarized as follows:
\begin{enumerate}[leftmargin=*,topsep=0pt,partopsep=1ex,parsep=1ex]\itemsep=-4pt
    \item \textbf{Proposed Algorithm for Regression.} We propose the Langevin Gradient Descent (LGD) Algorithm for regression in Algorithm \ref{alg:lgd}. The algorithm makes Langevin-diffusion inspired updates to estimate the mean prediction of a posterior distribution corresponding to the loss function and a regularizer. As opposed to prior work \citep{enet_2022, balcan_23, goyal_25} that restrict their attention to estimators of the parameters $w$, we propose a consolidation procedure that directly outputs an estimate of $y_v$ for some test/validation input $X_v$ based on the training data. We illustrate the significance of our consolidation step in Theorem \ref{thm:lgd_conv}
    which shows the Bayes' optimality of our procedure for the squared loss. We note further, that the proof of the theorem does not assume that training and test datasets are sampled i.i.d., meaning that the Bayes' optimality of the procedure holds even under distribution shift.
    \item \textbf{Proposed Algorithm for Classification.} We propose the variant of LGD for classification in Algorithm \ref{alg:lgd_class}. The algorithm makes Langevin-diffusion inspired updates followed by a consolidation procedure to estimate the posterior probabilities of different output classes. We demonstrate the closeness of our estimates to the true posterior probabilities in Theorem \ref{thm:lgd_conv_class}. Similar to above, the proof of the theorem does not assume that training and test datasets are sampled i.i.d., meaning that the Bayes' optimality of the procedure holds even under distribution shift.
    \item \textbf{Pseudo-Dimension bounds on Learning Hyperparameters for Regression.} We utilize the Goldberg-Jerrum framework for computing the pseudo-dimension of a function class as introduced in \citet{gj_95} to compute the pseudo-dimension of the function class corresponding to validation loss values in terms of the hyperparameters $\eta,\theta$ for regression problems in Theorem \ref{thm:pdim_comp}. While \citet{balcan_23} compute pseudo-dimension for linear regression with elastic net, we extend their results to any regression with a convex objective and any parametrized family of regularizers. Further, while \citet{balcan_23, goyal_25} assume perfect optimization of the elastic net, we study optimization using our proposed Algorithm \ref{alg:lgd}, and hence account for errors in optimization. Since we study the optimization procedure directly, we are able to give guarantees on any loss function that satisfy our assumptions, and do not rely on closed form solutions. Despite these generalizations, we obtain a pseudo-dimension bound that is linear (up to logarithmic terms) in the product of the number of hyperparameters $h$ and the mixing time of the algorithm, which is typically linear in the number of parameters $d$. We thus obtain a pseudo-dimension bound of $O(dh\log(d))$, which recovers the pseudo-dimension bounds of \citet{balcan_23} for the case of linear regression on elastic net up to logarithmic terms. 
    \item \textbf{Pseudo-Dimension bounds on Learning Hyperparameters for Classification.} Similar to our results on regression, we compute the pseudo-dimension of the function class corresponding to validation loss values in terms of the hyperparameters $\eta,\theta$ for classification problems in Theorem \ref{thm:pdim_class}. For our algorithm which is $\epsilon$-approximate with probability $\ge 1-\delta$ for some $\delta>0$ and a total of $(n+n_v)$ training and validation samples per task, we obtain a pseudo-dimension bound of the form $O\left(dh\log(d(n+n_v))\log(n_v)\left(\frac{1}{\epsilon^2} \log(1/\epsilon) + \frac{1}{\epsilon^4}\log(2/\delta)\right)\right)$. We compare our results with results from \citet{balcan_23}, where the authors compute pseudo-dimension bounds for logistic regression using either a $l1$ or $l2$ penalty, thus allowing for only $h=1$ hyperparameters. We extend their results to any classification problem with a convex objective and any parametrized family of regularizers. Further, \citet{balcan_23} consider an $\epsilon$-approximate algorithm that does not generalize beyond logistic regression. Despite these restrictions, \citet{balcan_23} suggest a pseudo-dimension bound of $O((n+n_v)^2 + \log(1/\epsilon))$, which is significantly worse from our bounds in terms of the dependence on number of samples. We however incur a degradation in terms of the dependence on $d$ and $\epsilon$, which we attribute to our choice of a sampling-based algorithm as opposed to an optimization-based algorithm. We further show that the bounds for regularized logistic regression can be improved to $O(\log(n_v/\epsilon))$ in Appendix \ref{sec:rlr_pdim}. Another notable difference is that our bounds analyze hardware-aware procedures that account for hardware approximations of the softmax function implemented as piecewise polynomials.
    \item \textbf{Generalization of Meta-Learned Algorithm to Optimal Algorithms.} We show generalization guarantees on estimating the optimal $\eta, \theta$ from previously seen tasks sampled similar to the target task in Theorem \ref{thm:gen_guarantee}. We then use these results to obtain concentration bounds of the ERM hyperparameters to the Bayes' optimal solution. We show that for well-specified regression problems with squared loss having log-strongly-concave priors, $T=O_\delta\left(\frac{dh\log(d)}{\epsilon^6}\right)$ tasks are sufficient to learn hyperparameters that achieve a generalization gap of less than $\epsilon$ from the Bayes' optimal solution with probability $\ge 1-\delta$ in Theorem \ref{thm:erm_bayes_conv}, where $O_\delta$ suppresses the dependence on $\delta$. 
    While typical generalization error bounds (for example, \citep{balcan_23, goyal_25}) only suffer from an $\epsilon$-dependence of $O(1/\epsilon^2)$, we note that such work assumes perfect optimization of the loss functions, and fail to achieve any concentration to the Bayes' optimal if the prior density is not Gaussian, Laplace or a product of the two.
    We show similar bounds for well-specified classification problems in Theorem \ref{thm:erm_bayes_class}. This further extends the treatment of prior work \citep{balcan_23, goyal_25} that fail to show convergence to Bayes' optimal solutions outside of Gaussian and Laplace priors while also assuming perfect optimization.
    \item \textbf{Empirical Evidence.} Finally, we present empirical evidence on the success of the proposed Langevin Gradient Descent (LGD) algorithm and meta learning hyperparameters for this algorithm in Section \ref{sec:expts}. We observe that LGD greatly out-performs plain gradient descent with no regularization on few-shot learning for linear regression, and outperforms an oracle gradient descent algorithm with perfect regularization on an atypical prior with mean $\neq$ mode. Compared to our results on linear regression, our results on classification show an even wider range of sample sizes for which few-shot learning using our techniques is advantageous due to the increased difficulty of the tasks. For the atypical prior, we additionally observe that the meta-learned algorithm beats the theoretical optimal algorithm with the same number of iterations, and performs on par with the theoretically best algorithm with $10\times$ the number of iterations for linear regression. For logistic regression, the meta-learned algorithm beats both the theoretical optimal algorithm with the same and with $10\times$ the number of iterations, again due to the increased difficulty of classification.
\end{enumerate}

\subsection{Related Work}\label{sec:rw}
Hyperparameter optimization is a well-studied problem in machine learning literature. Well-studied and widely adopted techniques include random search~\cite{JMLR:v13:bergstra12a}, Bayesian optimization~\cite{snoek2012practicalbayesianoptimizationmachine,NIPS2011_86e8f7ab}, population-based training~\cite{jaderberg2017populationbasedtrainingneural}, and bandit algorithms~\cite{li2018hyperbandnovelbanditbasedapproach}. Another commonly studied approach for hyperparameter optimization and meta learning is to formulate them as bi-level optimization problems~\citep{finn2017modelagnosticmetalearningfastadaptation,nichol2018firstordermetalearningalgorithms, franceschi2018bilevelprogramminghyperparameteroptimization, rajeswaran2019metalearningimplicitgradients, liu2021investigatingbileveloptimizationlearning}. 

Our approach falls under the \textit{data-driven algorithm design} framework of \citet{balcan2020data} where the authors suggest using data for provable tuning of hyperparameters in algorithms. This approach has been successful for designing various learning algorithms such as in \citet{bartlett_linalg_22, enet_2022, balcan_23, goyal_25}. \citet{enet_2022, balcan_23} give pseudo-dimension based generalization guarantees on tuning the elastic net. \citet{goyal_25} give a data-dependent analysis using Rademacher complexities that improves upon the analysis of \citet{balcan_23} for ``nice" distributions, such as when samples within each task are i.i.d.\ draws from broad well-studied classes of distributions including sub-Gaussians, and gives bounds that get tighter with increasing number of samples per task. In this work, we generalize the setting of the elastic net and provide guarantees on learning descent algorithms instead of just regularization coefficients. 
We additionally show that this provides the optimal regularization approach for any log-strongly-concave prior (as opposed to just Gaussian and Laplace). 
The consolidation step of our algorithm, involving averaging the predictions of the final few iterates, can be thought of as similar to Polyak averaging \citep{polyak_92, grattafiori2024llama3herdmodels}, which is an averaging of the final few iterates of Stochastic Gradient Descent. 
On the statistical front, our approach falls under the $g$-modelling framework for Empirical Bayes methods, a connection explored in \citet{goyal_25}.

While we assume access to the ERM solution for hyperparameters in this paper,
finding hyperparameters through gradient descent is a popular approach in empirical literature. \citet{gd_by_gd_16} suggest a similar approach to ours where they model the update step of gradient descent using an LSTM, and propose learning the parameters of the LSTM by meta learning but do not give any theoretical guarantees. This is different from our approach where we suggest Langevin-style updates that use the provided loss function with a regularizer.
The popular MAML algorithm \citep{finn2017modelagnosticmetalearningfastadaptation} simplified their approach by only optimizing for initial neural network parameters that optimize well on downstream tasks. Other approaches such as the hypergradient descent algorithm \citep{hyperparam_grad_dougal15, hypergrad_baydin18} and the recently proposed metagradient descent algorithm \citep{engstrom2025optimizingmltrainingmetagradient} also suggest optimizing over hyperparameters by approximating the gradients over the hyperparameters. DataRater \citep{datarater_calian25} suggest a similar approach for assigning optimal weights to training samples using meta learning. On the theory side,~\cite{pmlr-v97-balcan19a, NEURIPS2019_f4aa0dd9} have shown provable guarantees for MAML-like meta learning algorithms where the initialization of the weights for each task is meta-learned. Here, a geometric notion of task similarity is assumed (the optimal weights for each task are geometrically \textit{close}), as opposed to a prior-based assumption as in our work, and techniques from online convex optimization are used to obtain regret upper bounds. Learning hyperparameters for optimization also fall under the framework of learning to optimize \citep{l20_chen22}. Similar to our algorithms, Langevin dynamics are generally used a central tool for scalable posterior sampling in Bayesian inference. We explore this connection in further detail in Appendix \ref{sec:add_rel_work}.

\section{Background}\label{sec:bckg}
We begin our discussion with relevant background on statistical concentration inequalities and generalization error bounds using pseudo-dimension. We discuss the Goldberg-Jerrum framework that is the primary tool used in this paper to calculate the pseudo-dimension of the validation loss function class. Finally, we present a brief discussion on sampling using the Unadjusted Langevin Algorithm.

\subsection{Concentration Inequalities and Uniform Convergence}
We begin with Hoeffding's inequality, which shows that the mean of random variables concentrates exponentially fast around their  expected value.
\begin{theorem}[Hoeffding's inequality~\citep{all_of_stats}]\label{thm:Hoeffding}
For random numbers $Z_1, \ldots, Z_N$ sampled i.i.d., denote $\overline{Z}_N = \frac{\sum Z_i}{N}$ and $\E[Z_i] = \mu$. The following hold given that $Z_i\in [0,C]$:
\begin{enumerate}
    \item \begingroup\abovedisplayskip=0pt\belowdisplayskip=0pt\begin{align*}
        \Pr(|\overline{Z}_N - \mu| \geq t) \leq 2\exp{\left(\frac{-2Nt^2}{C^2}\right)}
    \end{align*}\endgroup
    \item \begingroup\abovedisplayskip=0pt\belowdisplayskip=0pt\begin{align*}
        \Pr(\overline{Z}_N - \mu \geq t) \leq \exp{\left(\frac{-2Nt^2}{C^2}\right)}
    \end{align*}\endgroup
    \item For any $\delta>0$, with probability $\geq 1-\delta$,\begingroup\abovedisplayskip=0pt\belowdisplayskip=0pt
    \begin{align*}
        \overline{Z}_N \leq \mu + C\sqrt{\frac{\ln{1/\delta}}{2N}}
    \end{align*}\endgroup
\end{enumerate}
\end{theorem}

We present a version of the Bernstein's inequality that shows tighter concentration bounds than the Hoeffding's inequality if the variance of i.i.d. observations is known.

\begin{theorem}[Bernstein's inequality~\citep{ML_theory_Shwartz_14}]\label{thm:Bernstein}
For random numbers $Z_1, \ldots, Z_N$ sampled i.i.d., denote $\overline{Z}_N = \frac{\sum Z_i}{N}$, $\E[Z_i] = \mu$ and $\var(Z_i) = V$. Given that $Z_i\in [-C,C]$ for $t>0$:
\begin{align*}
        \Pr(\overline{Z}_N - \mu \geq t) \leq \exp{\left(\frac{-3Nt^2}{2(3V+Ct)}\right)}
    \end{align*}
\end{theorem}

While concentration inequalities show concentration of means of random variables, we are often interested in \textit{uniform convergence}, that show that means of functions of random variables concentrate to their expected values \textit{uniformly over all functions in a class}. Notationally, we are interested in calculating $\sup_{g\in\mcG}\left|\frac{1}{N}\sum g(Z_i) - \E[g(Z)]\right|$ for some function class $\mcG$. We provide these guarantees for our algorithm using Pseudo-dimension, as defined below.
\begin{definition}[\citep{nn_theory_bartlett_99}]
    Let $\mcG$ be a set of functions mapping from $\mathcal{Z}$ to $\R$, and suppose that $S= \{z_1,...,z_s\}\subseteq \mathcal{Z}$. Then $S$ is pseudo-shattered by $\mcG$ if there
are real numbers $r_1,...,r_s$ such that for each $b \in \{0,1\}^s$ there is a function $f_b$ in $\mcG$ with
$sign(f_b(z_i) - r_i)=b_i$ for $i\in[s]$. We say that $r= (r_1,...,r_s)$ witnesses the shattering. We say that
$\mcG$ has pseudo-dimension $Pdim(\mcG)$ if $Pdim(\mcG)$ is the maximum cardinality of a subset $S$ of $\mathcal{Z}$ that is pseudo-shattered
by $\mcG$. If no such maximum exists, we say that $\mcG$ has infinite pseudo-dimension.
\end{definition}

Calculating the pseudo-dimension of a function class allows us to show uniform convergence results of the form below:
\begin{theorem}[\citep{nn_theory_bartlett_99}]\label{thm:pdim_generalization}
    Suppose $\mcG$ is a class of real-valued functions with range in $[0,C]$ and finite $Pdim(\mcG)$. Then for any $\epsilon > 0$ and $\delta \in (0,1)$, for any distribution $\mcD$ and for any set $S$ of $s = O(\frac{C^2}{\epsilon^2}(Pdim(\mcG) + \log(1/\delta)))$ samples drawn from $\mcD$, w.p. at least $1-\delta$, we have 
\begin{align*}
    \left|\frac{1}{s}\sum_{z_i\in S} g(z_i) - \E_{z\sim \mathcal{D}}[g(z)]\right| \le \epsilon, \quad \text{for all } g\in\mcG.
\end{align*}
\end{theorem}

\subsection{Goldberg-Jerrum Framework}
The Goldberg-Jerrum framework \citep{gj_95} is a popular tool for computing the Pseudo dimension of a function class.
\begin{definition}[GJ algorithm \citep{bartlett_linalg_22, balcan_23}]
    A GJ algorithm $\Gamma$ operates on real-valued inputs, and can perform two types of operations:
    \begin{itemize}
        \item Arithmetic operators of the form $\nu'' = \nu \odot\nu'$, where $\odot\in \{+,-,\times, \div\}$ for $\nu,\nu',\nu''\in \R$
        \item Conditional statement of the form ``if $\nu\ge 0$\ldots else \ldots".
    \end{itemize}
    In both cases, $\nu$ and $\nu'$ are either inputs or values previously computed by the algorithm.  
\end{definition}

\begin{definition}[\citep{bartlett_linalg_22, balcan_23}]
    The \textbf{degree} of a GJ algorithm is the maximum degree of any rational function it computes of the inputs. The \textbf{predicate complexity} of the GJ algorithm is the number of distinct rational functions that appear in its conditional statements.
\end{definition}

\begin{lemma}[\citep{bartlett_linalg_22, balcan_23}]\label{lem:gj_pdim}
    Suppose that each function $g\in\mcG$ is specified by $h$ real parameters. Suppose that for every $x\in\mcX$ and $r\in\R$, there is a GJ algorithm $\Gamma_{x,r}$ that given $g\in\mcG$, returns ``true" if $g(x)\ge r$, and ``false" otherwise. Assume that $\Gamma_{x,r}$ has degree $\Delta$ and predicate complexity $\Lambda$. Then, $Pdim(\mcG) = O(h\log(\Delta\Lambda))$.
\end{lemma}
Note that both $\Delta,\Lambda\ge1$ above since the query $g(x)\ge r$ is itself a conditional statement of degree 1 in $g(x)$. Typically, $g(x)$ will have a degree higher than $1$ in $x$.

\subsection{Langevin Algorithms}
Langevin diffusion is a popular technique to sample from a distribution with a known score function (gradient of log of pdf). For a distribution $\pi \propto \exp(-U)$ defined over $\R^d$, Langevin diffusion follows a continuous process governed by the following Stochastic Differential Equation (SDE):
\begin{align*}
    \dd z_t = -\nabla U(z_t)\dd t + \sqrt{2}\dd \xi_t,
\end{align*}
where $\xi_t$ corresponds to Brownian motion in $\R^d$. We use a discrete approximation of this continuous process called the homogeneous Unadjusted Langevin Algorithm (ULA) which performs the following iterative updates:
\begin{align*}
    z_{(k+1)} = z_{(k)} - \eta\nabla U+\sqrt{2\eta} \xi_k;\quad \xi_k\sim \mcN(0,\mathbb{I}).
\end{align*}

We will use results from \citep{durmus_ula_19} which provide a detailed analysis of the ULA. To begin, we note the following two assumptions:
\begin{assumption}\label{ass:grad_lip}
    The function $U$ is continuously differentiable on $\R^d$ and gradient Lipschitz: there exists $\beta\ge 0$ such that for all $z_1,z_2\in\R^d$, $\|\nabla U(z_1)-\nabla U(z_2)\| \le \beta\|z_1-z_2\|$.
\end{assumption}
\begin{assumption}\label{ass:strong_conv}
    $U$ is strongly convex, i.e. there exists $\alpha > 0$ such that for all $z_1,z_2\in \R^d$,
    \begin{align*}
        U(z_2) \ge U(z_1) + \langle\nabla U(z_1), z_2-z_1\rangle + (\alpha/2)\|z_1-z_2\|^2.
    \end{align*}
\end{assumption}
Note that for both of these assumptions to hold, $\alpha\le \beta$. We now discuss some important results about the homogeneous ULA with $\eta \le 1/(\alpha+\beta)$ and $\kappa = \frac{2\alpha\beta}{\alpha + \beta}$. 


The following Proposition shows convergence guarantees on the empirical average of a function $g$ applied to the final few iterates of the homogeneous ULA algorithm. Define $\|g\|_{Lip}$ to be the Lipschitz constant of $g$ and $\hat{\pi}_b^B(g)$ as the empirical average of $g$ applied to the last $b$ iterates after $B$ burn-in steps. Thus,
\begin{equation}
    \hat{\pi}_b^B(g) = \frac{1}{b}\sum_{i=B+1}^{b+B} g(z_{(i)}).
\end{equation}

\begin{proposition}\label{prop:bounds}
    If we run a homogeneous ULA with potential function $U$ that satisfies the assumptions \ref{ass:grad_lip} and \ref{ass:strong_conv} and $\eta=O(\epsilon^2/(\|g\|_{Lip}^2d))$, $B = \Omega\left(\frac{d\|g\|_{Lip}^2}{\epsilon^2}\log\left(\frac{(d+\|z-z^*\|^2)\|g\|_{Lip}^2}{\epsilon^2}\right)\right)$
    and $b$ that satisfies $b\eta \ge \frac{128\|g\|_{Lip}^2\log(1/\delta)}{\epsilon^2\kappa^2}$,
    \begin{align*}
        \var(\hat{\pi}_b^B)&\le \frac{\epsilon^2}{8\log(1/\delta)} \quad \text{and,}\\
        |\hat{\pi}_b^B(g) - \pi(g)| &\le \epsilon \quad \text{w.p. $\ge 1-\delta$}.
    \end{align*}
\end{proposition}
\begin{proof}
    Deferred to Appendix \ref{sec:asymp_lgd}.
\end{proof}

These and other relevant results on the ULA are proved in Appendix \ref{sec:asymp_lgd}.

\section{Problem Setting and Notation}\label{sec:setting}
We follow the notation of \citet{goyal_25}. We denote a regression or classification task as the quadruple of training and validation datasets $(X, y, X_v, y_v)$.
Each input $x$ in either the training or validation set lies in $\mathbb{R}^{\dx}$. The matrices $X \in \mathbb{R}^{n \times \dx}$ and $X_v \in \mathbb{R}^{n_v \times \dx}$ collect the $n$ training inputs and $n_v$ validation inputs, respectively. The outputs for regression are real-valued, with $y \in \mathbb{R}^n$ and $y_v \in \mathbb{R}^{n_v}$. The outputs for classification tasks are categorical, with $y \in [m]^n$ and $y_v \in [m]^{n_v}$, where $m$ is the number of classes.
We study regression with respect to weights $w \in \mathbb{R}^{\dw}$ for the function $f(x; w) : \mathbb{R}^{\dx} \times \mathbb{R}^{\dw} \to \mathbb{R}$.
As a special case, $f(x; w) = x^\intercal w$ recovers standard linear regression. Similarly, we study classification with respect to weights $w \in \mathbb{R}^{\dw}$ for the function $f(x; w) : \mathbb{R}^{\dx} \times \mathbb{R}^{\dw} \to \mathbb{R}^m$. When applying $f$ to a dataset, we interpret this operation row-wise; for instance, $f(X; w)$ denotes the vector obtained by evaluating $f(\cdot; w)$ on each row of the matrix $X$. Additionally, we will denote the $i$th entry of an input/output matrix with a superscript in parenthesis. For example, $X^{(i)}$ is the $i$th input in the training set.

Depending on whether the task is regression or classification, the output space $\mcY$ can be $\R$ or $[m]$ respectively, where $m$ is the number of classes.
We denote the predictions of a learning algorithm on $X_v$ as $\hat{f}_v(X_v,X,y;\phi,\xi)$, for $\hat{f}_v:\R^{n_v\times \dx}\times\R^{n\times \dx}\times \R^n \times \R^\dhp\times \R^e \rightarrow\mcY^{n_v}$ where $\phi\in\R^\dhp$ represent the hyperparameters of the learning algorithm, and $\xi\in\R^e$ is the randomness of the algorithm sampled as $\xi\sim\mcD_{E}$.
When the data and hyperparameters are obvious from context, we will use $\hat{f}_v$
as a shorthand. We will overload the notation for loss functions to be $l: \R\times \R\rightarrow\R$ or $l: [m] \times \R^m\rightarrow\R$ depending on whether the task is regression or classification respectively. We can then write the validation loss as:
\begin{align*}
    \lv(\phi,\xi, (X,y,X_v,y_v)) = \frac{1}{n_v}\sum_i l(y_v^{(i)},\hat{f}_v(X_v^{(i)},X,y;\phi,\xi)).
\end{align*}

We consider the problem of finding the optimal hyperparameters, $\phi^*$ from $T$ tasks. 
We assume that each task is sampled from a distribution $\mcD$ such that $(X,y,X_v,y_v) \sim \mcD$ and denote
the support of $\mcD$ as $\mathcal{P}$, the set of all tasks.
We define a problem instance as an ordered set of tasks sampled i.i.d.:
\begin{align*}
    S = \{(X^t,y^t,X^t_{v}, y^t_{v})_{t=1}^T: (X^t,y^t,X^t_{v}, y^t_{v})\stackrel{\text{i.i.d.}}{\sim} \mcD, \forall t\in[T]\}.
\end{align*}
The superscript $t$ without any parentheses denotes the task, so that $X^{t(i)}$ is the $i$th input in the training set of the $t$th task.
Define the expected validation error as follows:
\begin{align*}
    \lv(\phi) &= \E_{(X,y,X_v,y_v)\sim\mcD,\xi\sim\mcD_E}[\lv(\phi, \xi, (X,y,X_v,y_v))]\\
    &= \E_{(X,y,X_v,y_v)\sim\mcD,\xi\sim\mcD_E}\left[\frac{1}{n_v}\sum_i l(y_v^{(i)},\hat{f}_v(X_v^{(i)},X,y;\phi,\xi))\right],
\end{align*}
which allows us to define $\phi^*$ as $\phi^* = \argmin_\phi \lv(\phi)$.
We will similarly define the empirical validation error using randomness $\xi^{t}$ for task $t$ as follows:
\begin{align*}
    \lv(\phi, \xi, S) = \frac{1}{T}\sum_t\frac{1}{n_v}\sum_i l(y_v^{t(i)},\hat{f}_v(X_v^{t(i)},X^t,y^t;\phi,\xi^t)),
\end{align*}
where we represent each individual task with a superscript. This allows us to define the ERM estimator of $\phi$ as $\phi_{ERM}(\xi)$ which satisfies $\phi_{ERM}(\xi) = \argmin_\phi \lv(\phi, \xi, S)$. Similarly, we define $\phi^*(\xi) = \argmin_\phi \E_S[\lv(\phi, \xi, S)]$. Note that both $\phi_{ERM}(\xi)$ and $\phi^*(\xi)$ depend on the randomness $\xi$ used for learning the hyperparameters. $\phi^*$ as defined above is the hyperparameter that performs the best on average over randomness of the algorithm. 

Some of our results assume the existence of a true data-generating prior $\pi$ over the parameter $w$, such that the distributions of $y$ and $y_v$ only depend on $X,w$ and $X_v,w$ respectively. We define $B_\pi$ as the Bayes' risk, which is equal to $l_v(\phi)$ in these cases. Formally for an estimator $\hat{f}_v(X_v,X,y;\phi,\xi)$, and the true data-generating prior $\pi$ over the parameter $w$,
\begin{align*}
    B_\pi(\hat{f}_v) &= \lv(\phi) = \E_{(X,y,X_v,y_v),\xi}[\lv(\phi, \xi, (X,y,X_v,y_v))]\\
    &= \E_{X,X_v,\xi}[\E_{w\sim \pi}[\E_{y|(X,w), y_v|(X_v,w)}[\lv(\phi, \xi, (X,y,X_v,y_v))]]]\\
    &= \E_{X,y,X_v,\xi}[\E_{w|X,y}[\E_{y_v|X_v,w}l(y_v,\hat{f}_v(X_v,X,y;\phi,\xi))]].
\end{align*}
The Bayes' optimal estimator minimizes the Bayes' risk such that $\hat{f}_v^{opt} = \argmin_{\hat{f}_v} B_\pi(\hat{f}_v)$. 
For Langevin updates we employ gradients of different functions. The partial derivative of a scalar function $g(z)$ with respect to a scalar input $z$ is written as $\partial_z g$. If the input is a vector, $z\in\R^{N}$ for some $N\in\mathbb{N}$, we write the gradient as $\nabla_zg$, where the $i^{th}$ entry is a partial derivative, $\nabla_zg[i] = \partial_{z[i]}g$. Lastly, if the output of the function $g\in\R^{N_1}$ and $z\in\R^{N_2}$ for $N_1,N_2\in\mathbb{N} $, we write the Jacobian matrix as $J_zg\in\R^{N_1\times N_2} $, where the $(i,j)-$th entry is a partial derivative, $J_zg[i,j] = \partial_{z[j]} g[i] $.

We present closeness of optimal LGD algorithm to Bayes' optimal solutions for well-specified problems in Section \ref{sec:lgd}. We then present generalization guarantees on estimating $\phi^*$ using $\phi_{ERM}(\xi)$ in Section \ref{sec:guarantees} and show convergence of the outputs of LGD algorithm using hyperparameters $\phi_{ERM}(\xi)$ to Bayes' optimal outputs.


\section{Langevin Gradient Descent}\label{sec:lgd}
We propose the Langevin Gradient Descent (LGD) algorithm in Algorithms \ref{alg:lgd} and \ref{alg:lgd_class} for regression and classification tasks respectively. Notably, while typical regression algorithms compute the minimizer of regularised loss, the LGD performs Langevin-style updates on the regularised loss, motivated by the Unadjusted Langevin Algorithm as discussed in Section \ref{sec:bckg}. Intuitively, when the regularizer represents the score of a prior over parameters $w$, the Langevin style updates should mimic the ULA over the posterior distribution. 

For regression tasks, Algorithm \ref{alg:lgd} averages the final few iterates of the ULA step to achieve an estimate of the mean predictions using the posterior. 
Theorem \ref{thm:lgd_conv} shows that this procedure achieves Bayes' optimality for the squared loss given optimal hyperparameters for the gradient of regulariser function $r(.)$. This is in sharp contradiction to gradient descent algorithms, where we only obtain the mode of the posterior distribution, which is not necessarily the Bayes' optimal solution, even with optimal hyperparameter configurations.

For classification tasks, Algorithm \ref{alg:lgd_class} treats the outputs of the function $f$ as logits and averages the final few iterates of ULA to output an estimate of the posterior probability (not the logits of probabilities). We denote the estimate of the posterior probabilities as $\hat{F}_v$ to distinguish from $\hat{f}_v$, which would be a posterior estimator of the function $f$. We use the function $\smax(.):\R^m\rightarrow\Sigma_{m}$, where $\Sigma_m=\{v\in[0,1]^m|\sum_jv[j]=1\}$ is the probability simplex, in the algorithm, which we assume is a piecewise-rational function 
to accurately model hardware implementations of the $\exp(.)$ function, which are piecewise polynomial approximations \citep{fp_arithmetic_muller_18}. This also eases calculations in Section \ref{sec:guarantees}. Similar to other notation in the paper, the function $\smax(.)$ applied to a matrix of size $n\times m$ for any $n$ is equal to the matrix obtained by applying $\smax$ to each row of the input matrix and composing the results. 
Since softmax is 1/2-Lipschitz as discussed in \citet{nair2026softmax}, we select the approximation of softmax $\smax(.)$ such that each entry is 1-Lipschitz. Formally, for all $x,x'$ in $\mathcal{P}$ and $w,w'\in\R^\dw$, the following assumption should be satisfied.
\begin{assumption}\label{ass:e_approx}
    $\left|\smax(f(x;w))[j] - \smax(f(x';w'))[j]\right| \le \|f(x;w)-f(x';w')\|,\; \forall j \in [m]$.
\end{assumption}

We make the following assumptions on the loss function $l(\cdot,\cdot)$ the gradient of regulariser $r(\cdot;\cdot)$ and the function $f(\cdot;\cdot)$, valid for all training pairs $X,y$ in $\mathcal{P}$, $w,w'\in\R^\dw$ and $\theta\in\R^{\dhp}$:
\begin{assumption}[]\label{ass:loss_lips}
    $l(y,f(X;w))$ is gradient-Lipschitz and convex in $w$.
\end{assumption}
\begin{assumption}\label{ass:reg_lips}
    $r(w;\theta)$ admits a positive definite Jacobian with respect to $w$ and is Lipschitz in $w$ satisfying $\|r(w;\theta)-r(w';\theta)\|\le \mu_r\|w-w'\|$ for some $\mu_r$.
\end{assumption}
\begin{assumption}\label{ass:f_lips_w}
    $f(x;w)$ is Lipschitz in $w$ for all $x$, such that $\|f(x;w)-f(x;w')\|\le \|f\|_{Lip}\|w-w'\|$.
\end{assumption}

\begin{figure*}[ht]
\centering

\begin{minipage}[t]{0.48\linewidth}
\begin{algorithm}[H]
\caption{$LGD$-Regression}
\label{alg:lgd}
\begin{algorithmic}[1]

\REQUIRE Learning rate $\eta$, regulariser parameters $\theta$, training data $(X,y)$, regression function $f(.)$, gradient of regulariser $r(.)$, Langevin noise $\xi$, 
burn-in steps $B$,
averaging steps $b$,
and validation input $X_v$.


\STATE $w_{(0)} \gets 0$

\vspace{0.5\baselineskip}
\textbf{Step 1}: ULA starting at $w_{(0)}$
\FOR{$k = 0$ \textbf{to} $B+b-1$}
    \STATE $g_{(k)} \gets  
    \nabla_w f(X; w_{(k)})\partial_{y'}l(y,y'=f(X;w_{(k)}))$
    
    \hspace{3em} $ + \; r(w_{(k)};\theta)$
    \STATE $w_{(k+1)} \gets w_{(k)} - \eta\, g_{(k)} + \sqrt{2\eta}\,\xi_{(k)}$
\ENDFOR

\vspace{0.5\baselineskip}
\textbf{Step 2}: Empirical mean of predictions
\STATE $\hat{f}_v \gets \frac{1}{b} \sum_{k=B+1}^{B+b} f(X_v; w_{(k)})$


\vspace{0.5\baselineskip}
\STATE \textbf{return} $\hat{f}_v$

\end{algorithmic}
\end{algorithm}
\end{minipage}
\hfill
\begin{minipage}[t]{0.48\linewidth}
\begin{algorithm}[H]
\caption{$LGD$-Classification}
\label{alg:lgd_class}
\begin{algorithmic}[1]

\REQUIRE Learning rate $\eta$, regulariser parameters $\theta$, training data $(X,y)$, model $f(.)$, gradient of regulariser $r(.)$, Langevin noise $\xi$, 
burn-in steps $B$,
averaging steps $b$,
and validation input $X_v$.


\STATE $w_{(0)} \gets 0$

\vspace{0.5\baselineskip}
\textbf{Step 1}: ULA starting at $w_{(0)}$
\FOR{$k = 0$ \textbf{to} $B+b-1$}
    \STATE $g_{(k)} \gets  
    J_w (f(X; w_{(k)}))^\intercal\nabla_{p'}l(y,p'=f(X;w_{(k)}))$
    
    \hspace{3em} $ + \; r(w_{(k)};\theta)$
    \STATE $w_{(k+1)} \gets w_{(k)} - \eta\, g_{(k)} + \sqrt{2\eta}\,\xi_{(k)}$
\ENDFOR

\vspace{0.5\baselineskip}
\textbf{Step 2}: Empirical mean of probabilities
\STATE $\hat{F}_v \gets \frac{1}{b} \sum_{k=B+1}^{B+b} \smax(f(X_v; w_{(k)}))$


\vspace{0.5\baselineskip}
\STATE \textbf{return} $\hat{F}_v$

\end{algorithmic}
\end{algorithm}
\end{minipage}

\end{figure*}


\begin{remark}[On the practicality of assumptions]
    We observe that the assumptions stated above all require Lipschitzness of some form, and note that the practicality of such an assumption is not immediate. To see the practicality, we note that Lipschitzness for several functions follows trivially after assuming boundedness of the inputs. And we can assume all inputs are bounded with high probability since real-world distributions typically follow well-behaved tail distributions \citep{kontorovich13,galvez24}. We would like to redirect the reader to \citet{goyal_25} for a detailed discussion. Further, Lipschitzness of $\smax$ can be intuitively justified since the function is an approximation of the softmax function, which is 1/2-Lipschitz as noted above.
\end{remark}

\begin{remark}[On the choice of polynomially approximating the softmax]
    Our choice of polynomially approximating the softmax is a distinct departure from prior work \citep{balcan_23}, that derives bounds on learning l1 and l2 regularization for binary logistic regression using the actual softmax. As we note in Section \ref{sec:guarantees}, the bounds of \citet{balcan_23} degrade sharply with the number of samples, which is counter-intuitive. While we do improve their bounds for the exact softmax in Appendix \ref{sec:rlr_pdim}, we note that the techniques of Appendix \ref{sec:rlr_pdim} do not immediately generalize to the general cases considered in the paper. The sharp degradation with the number of samples as observed by \citet{balcan_23} in fact a fundamental limitation of working with exponential functions \citep{gabrielov2004complexity}. Thus our assumption of a piecewise polynomial approximation leads to better bounds, along with being more faithful to hardware implementations. Finally, while we focus our work on giving bounds under hardware restrictions, we note that analogous bounds to ours can be derived by considering the existence of a piecewise polynomial approximation of the softmax, without assuming hardware restrictions. This can be done following the treatment of \citet{poly_approx_vitercik2020}, where the authors prove upper bounds on the Rademacher complexity of a class of algorithms that can be implemented approximately using piecewise polynomials.
\end{remark}

The theorem below establishes that Algorithm \ref{alg:lgd} computes estimators close to the Bayes' optimal for regression tasks.
\begin{theorem}\label{thm:lgd_conv}
    Assume a well-specified regression problem such that there exists a $\ws$ for which $y \sim \mcN(f(X;\ws), I_n)$ and $y_v \sim \mcN(f(X_v;\ws),I_{n_v})$. That is, there exists a ground truth $\ws$ for which the problem is exactly solvable up to Gaussian noise. Further assume that the ground truth $\ws \sim \mcD_W$. For the squared loss $l(y,y') = \|y-y'\|^2$, the Bayes' optimal predictions $\hat{f}_v$ satisfy:
    \begin{equation*}
        \hat{f}_v^{opt}(X_v) = \E_{\ws|X,y}[f(X_v;\ws)].
    \end{equation*}
    If the Assumptions \ref{ass:loss_lips}-\ref{ass:f_lips_w} are satisfied, $\theta$ is given such that $r(w;\theta) = -\nabla_w \log\mcD_W(w)$, 
    then for $\eta = O(\epsilon^2/(\|f\|_{Lip}^2\dw))$ and within $B = O\left(\frac{\dw\|f\|_{Lip}^2}{\epsilon^2}\log\left(\frac{(\dw+\|w_{(0)}-w_{min}\|^2)\|f\|_{Lip}^2}{\epsilon^2}\right)\right)$,
    the LGD algorithm as given in Algorithm \ref{alg:lgd} is $\epsilon$-close in expectation to the Bayes' optimal solution for each entry of $X_v$:
    \begin{align*}
        |\E[\hat{f}_v(X_v^{(i)})] - \E_{\ws|X,y}[f(X_v^{(i)};\ws)]| &\le \epsilon\; \forall i\in [n_v]\\
        \implies \|\E[\hat{f}_v] - \E_{\ws|X,y}[f(X_v;\ws)]\| &\le \sqrt{n_v}\epsilon,
    \end{align*}
    where $w_{min}$ is the mode of the posterior $w_{min} = \argmax_w Pr(w|X,y)$. 
    Further, $b= O\left(\frac{\|f\|_{Lip}^2\log(1/\delta)}{\eta\epsilon^2}\right)$  is sufficient to guarantee, for a particular entry of $X_v$ denoted $X_v^{(i)}$:
    \begin{align*}
        Pr(|\hat{f}_v(X_v^{(i)}) - \E_{\ws|X,y}[f(X_v^{(i)};\ws)]|\ge \epsilon)\le \delta.
    \end{align*}
    Alternatively,
    \begin{align*}
        Pr(\|\hat{f}_v - \E_{\ws|X,y}[f(X_v;\ws)]\|\ge \sqrt{n_v}\epsilon)\le n_v\delta.
    \end{align*}
\end{theorem}
\begin{proof}
We first prove the Bayes' optimality of the estimator $\hat{f}_v(X_v) = \E_{\ws|X,y}[f(X_v;\ws)]$. For any estimator $\hat{f}_v(\cdot)$ with randomness $\xi$, we write the Bayes' risk from Section \ref{sec:setting} as:
\begin{align*}
    B_\pi(\hat{f}_v) &=\E_{X,y,X_v,\xi}[\E_{w|X,y}[\E_{y_v|X_v,w}[\|\hat{f}_v(X_v,X,y) - y_v\|^2]]]\\
    &\ge \E_{X,y,X_v}[\E_{w|X,y}[\E_{y_v|X_v,w}[\|\E_\xi[\hat{f}_v(X_v,X,y)] - y_v\|^2]]], 
\end{align*}
by Jensen's inequality. We can thus focus on deterministic estimators in this discussion.
We note that $y_v$ is sampled from the same ground truth $w^*$ as $y$ so we can re-write this as: 
\begin{align*}
    B_\pi(\hat{f}_v) &= \E_{X,y,X_v}[\E_{w|X,y}[\E_{\epsilon_v}[\|\hat{f}_v(X_v,X,y) - f(X_v;w) - \epsilon_v\|^2]]]\\
    &= \E_{X,y,X_v}[\E_{w|X,y}[\|\hat{f}_v(X_v,X,y) - f(X_v;w)\|^2 + n_v\E[\|\epsilon_v\|^2]]]\\
    &= n_v\E[\|\epsilon_v\|^2] + \E_{X,y,X_v}[\E_{w|X,y}[\|\hat{f}_v(X_v,X,y) - f(X_v;w)\|^2]],
\end{align*}
where $n_v$ is the number of examples in the validation set $X_v$. Now,  we want to find the $\hat{f}_v$ that minimizes the Bayes error above, so we only focus on the second term. Since the mean of a distribution minimizes the expected value of squared distances, we can conclude that the minimizer satisfies: $\hat{f}_v(X_v,X,y) = \E_{\ws|X,y}[f(X_v;w^*)]$.

To show Bayes' optimality of our algorithm, we first note that the posterior satisfies $Pr(w|X,y)\propto Pr(y|X,w)Pr(w)$ so that $\nabla_w\log Pr(w|X,y) = -\nabla_w f(X;w)(f(X;w)-y) - r(w;\theta)$ from the given conditions. Thus, the ULA in Step 1 of Algorithm \ref{alg:lgd} is a ULA on the posterior distribution. Using results from Proposition \ref{prop:bounds}, we can say that for each element $X_v^{(i)}$ of $X_v$, $\left|\E[\hat{f}_v(X_v^{(i)})] - \E_{\ws|X,y}[f(X_v^{(i)};\ws)]\right| \le \epsilon$ for the given values of $B$ and $\eta$ by replacing the function $f(.)$ in the Proposition with $f(X_v^{(i)};.)$. Alternatively, $\|\E[\hat{f}_v] - \E_{\ws|X,y}[f(X_v;\ws)]\| \le \sqrt{n_v}\epsilon$. Similarly, for the chosen value of $b$, we can give the high probability bound $Pr(|\hat{f}_v(X_v^{(i)}) - \E_{\ws|X,y}[f(X_v^{(i)};\ws)]|\ge \epsilon)\le \delta$, or alternatively, $Pr(\|\hat{f}_v(X_v) - \E_{\ws|X,y}[f(X_v;\ws)]\|\ge \sqrt{n_v}\epsilon)\le n_v\delta$ by a union bound. 
\end{proof}

Similar to the above result, we show below that Algorithm \ref{alg:lgd_class} outputs probabilities close to the true softmax of the Bayes' optimal logits for classification. As a reminder to the reader, we will denote the logits by smaller-case notation $f$, and the probabilities by $\smax(f)$ or the upper-case notation $F$.
\begin{theorem}\label{thm:lgd_conv_class}
    Assume a well-specified classification problem such that: 
    \begin{align*}
        Pr(y^{(i)}=j|X^{(i)},w) = \smax(f(X^{(i)},w))[j] \quad \text{
    and} \quad  Pr(y_v^{(i)}  = j|X_v^{(i)},w) = \smax(f(X_v^{(i)},w))[j].
    \end{align*}
    Further assume that the ground truth $\ws \sim \mcD_W$. Assume the loss function $l(y=j,f) = -\log(\smax(f)[j])$ such that $l(y=j,f)$ is convex in all entries of $f$. 
    Then the Bayes' optimal outputs $\hat{f}_v^{opt}$ satisfy:
    \begin{align*}
        Pr(y_v^{(i)}=j|X_v^{(i)},X,y) = \smax(\hat{f}_v^{opt}(X_v^{(i)}))[j]
    \end{align*}
    If the Assumptions \ref{ass:e_approx}-\ref{ass:f_lips_w} are satisfied, $\theta$ satisfies $r(w;\theta) = -\nabla_w \log\mcD_W(w)$, 
    then for $\eta = O(\epsilon^2/(\|f\|_{Lip}^2\dw))$ and within $B = O\left(\frac{\dw\|f\|_{Lip}^2}{\epsilon^2}\log\left(\frac{(\dw+\|w_{(0)}-w_{min}\|^2)\|f\|_{Lip}^2}{\epsilon^2}\right)\right)$,
    the LGD algorithm as given in Algorithm \ref{alg:lgd_class} outputs probabilities that are $\epsilon$-close in expectation to the exponent of Bayes' optimal logits for each entry of $X_v$:
    \begin{align*}
        \left|\E\left[\hat{F}_v(X_v^{(i)})[j]\right] - \smax(\hat{f}_v^{opt}(X_v^{(i)}))[j]\right| &\le \epsilon\; \forall i\in [n_v], j\in[m]\\
        \implies \sum_{i\in [n_v], j\in[m]}\left|\E\left[\hat{F}_v(X_v^{(i)})[j]\right] - \smax(\hat{f}_v^{opt}(X_v^{(i)}))[j]\right| &\le  mn_v\epsilon,
    \end{align*}
    where $w_{min}$ is the mode of the posterior $w_{min} = \argmax_w Pr(w|X,y)$. 
    Further, $b= O\left(\frac{\|f\|_{Lip}^2\log(1/\delta)}{\eta\epsilon^2}\right)$  is sufficient to guarantee, for a particular entry of $X_v$ denoted $X_v^{(i)}$:
    \begin{align*}
        Pr\left(\left|\hat{F}_v(X_v^{(i)})[j] -  \smax(\hat{f}_v^{opt}(X_v^{(i)}))[j]\right|\ge \epsilon\right)\le \delta.
    \end{align*}
    Alternatively,
    \begin{align*}
        Pr\left(\sum_{i\in [n_v], j\in[m]}\left|\hat{F}_v(X_v^{(i)})[j] -  \smax(\hat{f}_v^{opt}(X_v^{(i)}))[j]\right|\ge mn_v\epsilon\right)\le mn_v\delta.
    \end{align*}
\end{theorem}
\begin{proof}
We first prove the Bayes' optimality of the estimator that satisfies:
\begin{align*}
    Pr(y_v^{(i)}=j|X_v^{(i)},X,y) = \smax(\hat{f}_v^{opt}(X_v^{(i)}))[j].
\end{align*}
For any estimator, $\hat{f}_v(\cdot)$ with randomness $\xi$, we write the Bayes' risk from Section \ref{sec:setting} as:
\begin{align*}
    B_\pi(\hat{f}_v) &=\E_{X,y,X_v,\xi}\left[\E_{w|X,y}\left[\E_{y_v|X_v,w}\left[l(y_v,\hat{f}_v(X_v,X,y;\xi))\right]\right]\right]\\
    &\ge \E_{X,y,X_v}[\E_{w|X,y}\left[\E_{y_v|X_v,w}\left[l(y_v,\E_\xi[\hat{f}_v(X_v,X,y;\xi)])\right]\right]], 
\end{align*}
by Jensen's inequality and convexity of $l$.
We can thus focus on deterministic estimators in this discussion.
We can again use convexity to note that 
the minimizer of $B_\pi(\hat{f}_v) = \E_{X,y,X_v}[\E_{y_v|X_v,X,y}[l(y,\hat{f}_v)]]$ is the zero of the gradient, $\nabla_f B_\pi(f)$, if there is such a zero. It is easy to see that for $\hat{f}_v$ that satisfies $Pr(y_v^{(i)}=j|X_v^{(i)},X,y) = \smax(\hat{f}_v^{opt}(X_v^{(i)}))[j]$,
$\nabla_f B_\pi(f)$ is the expected gradient of the score of $Pr(y_v|X_v,X,y)$, which is always zero. We show this below, where we first swap the expectation with the differentiation by noting the differentiability of $l(y_v=j,f)$ in $f$, and the derivative being an integrable function of $X,y,X_v$:
\begin{align*}
    \nabla_f B_\pi(f) &= \nabla_f \E_{X,y,X_v}\left[ \sum_{j\in[m]} l(y_v=j,f)Pr(y_v=j|X_v,X,y)\right]\\
    &= \E_{X,y,X_v}\left[ \nabla_f \left( \sum_{j\in[m]} l(y_v=j,f)Pr(y_v=j|X_v,X,y)\right)\right].
\end{align*}
We can now simplify the resulting expression as follows:
\begin{align*}
    \nabla_f B_\pi(f) &= \E_{X,y,X_v}\left[ \nabla_f \left( \sum_{j\in[m]} l(y_v=j,f)Pr(y_v=j|X_v,X,y)\right)\right]\\
    &= \E_{X,y,X_v}\left[\nabla_f \left(\sum_{j\in[m]} l(y_v=j,f)\exp(\log(Pr(y_v=j|X_v,X,y)))\right)\right]\\
    &= \E_{X,y,X_v}\left[\nabla_f \left(\sum_{j\in[m]} l(y_v=j,f)\exp(- l(y_v=j,\hat{f}_v))\right)\right]\\
    \implies \nabla_f B_\pi(f)|_{f=\hat{f}_v} &= \E_{X,y,X_v}\left[\nabla_f \left.\left(\sum_{j\in[m]} -\exp(- l(y_v=j,f))\right)\right|_{f=\hat{f}_v}\right] \\
    &= \E_{X,y,X_v}\left[\nabla_f \left(\sum_j -\smax(f)[j]\right)\right]= \E_{X,y,X_v}\left[\nabla_f (-1)\right] = 0.
\end{align*}

In the last step we use the fact that the codomain of $\smax$ is the probability simplex $\Sigma_m$ as defined at the start of Section \ref{sec:lgd}. This concludes the first part of the proof.

To show Bayes' optimality of our algorithm, we note similar to the proof of Theorem \ref{thm:lgd_conv} that the ULA in Step 1 of Algorithm \ref{alg:lgd_class} is a ULA on the posterior distribution. 
Now since $\smax$ is 1-Lipschitz from Assumption \ref{ass:e_approx}, we can say that $|\smax(f(X,w))[j] - \smax(f(X,w'))[j]| \le \|f(X,w)-f(X,w')\|\le \|f\|_{Lip}\|w-w'\| $.
Using results from Proposition \ref{prop:bounds}, we can say that for each element $X_v^{(i)}$ of $X_v$, and each class $j\in[m]$, $|\E[\hat{F}_v(X_v^{(i)})[j]] - \E_{w|X,y}[\smax(f(X_v^{(i)},w))[j]]|\le \epsilon $
for the given values of $B$ and $\eta$ by replacing the function $f(.)$ in the Proposition with $\smax(f(X_v^{(i)};.))[j]$. The convergence in expectation then follows by noting that $\smax(f(X_v^{(i)},w))[j] = Pr(y_v^{(i)}=j|X_v^{(i)},w) \implies \E_{w|X,y}[\smax(f(X_v^{(i)},w))[j]] = \E_{w|X,y}[Pr(y_v^{(i)}=j|X_v^{(i)},w)] = Pr(y_v^{(i)}=j|X_v^{(i)},X,y) $.

Similarly, for the chosen value of $b$, we can give the high probability bound, 
\begin{align*}
    Pr\left(\left|\hat{F}_v(X_v^{(i)})[j] -  \smax(\hat{f}_v^{opt}(X_v^{(i)}))[j]\right|\ge \epsilon\right)\le \delta.
\end{align*}
The desired results follow from a union bound.
\end{proof}

\subsection{Bayes Optimal Algorithms for General Loss Functions}
Results in the above Section establish closeness to the Bayes' optimal for the LGD algorithm under specific loss functions. In this Section, we briefly discuss extensions of these Algorithms for Bayes' optimality beyond the specific loss functions.

Theorem \ref{thm:lgd_conv} shows that LGD provides Bayes' optimal solutions for the squared loss if the problem is well-specified. We require two key facts to show this result: firstly, that the output distribution $y\sim \mcN(f(X;\ws), I_n)$ is Gaussian in the ground-truth model outputs, and the consolidation step in Step 2 of Algorithm \ref{alg:lgd} performs an empirical average. The Gaussian distribution implies that the squared loss is the negative log likelihood of the observed outputs, so that Step 1 of Algorithm \ref{alg:lgd} performs sampling from the posterior. The consolidation step computes the empirical mean, since the Bayes' optimal solution for the squared loss is the posterior mean. 

To consider an extension to a general loss function, we first need to revisit the well-specified assumption. Typical definitions of a well-specified regression problem \citep{offset_rad, goyal_25} only require that the expectation value of the output is represented by the model, $\E[y|X,w] = f(X;w)$. These works however, assume the learning objective is the squared loss. We propose the following definition of a well-specified regression problem that allows us to write the negative log-likelihood of the output as the loss function.

\begin{definition}[Well-Specified Regression Problem]
    A regression problem on the function $f(.;.):\R^{\dx}\times \R^{\dw} \rightarrow \R$ and a loss function $l(.,.):\R\times\R\rightarrow\R$ is well-specified if the following probability is well-defined (i.e., integrates to 1 over all $y\in\R$) and for parameters $w\in\R^{\dw}$:
    \begin{align*}
        Pr(y|X,w) = \exp(-l(y,f(X;w))).
    \end{align*}
\end{definition}

The above definition is a generalisation of the assumptions of Theorem \ref{thm:lgd_conv} such that the ULA in Step 1 of Algorithm \ref{alg:lgd} samples  $w$ from the posterior distribution $\mathcal{D}_{w|X,y}$. For a general loss function, the Bayes' optimal solution is not necessarily the posterior mean but the minimizer of $B_\pi(\hat{f}_v) = \E_{w|X,y}[\E_{y_v|X_v,w}[l(y_v, \hat{f}_v(X_v))]]$, which we may not have a closed form solution for. Given a sampling procedure to sample from $\mcD_{y_v|X_v,w}$, we can estimate $B_\pi(\hat{f}_v)$ for any prediction $\hat{f}_v$, and compute the minimizer using standard optimization techniques. Since it is non-trivial to have efficient sampling from both distributions $\mathcal{D}_{w|X,y}$ and $\mcD_{y_v|X_v,w}$ without reducing the problem to linear regression, we defer further discussion on regression problems to future work and restrict our attention to LGD as defined in Algorithm \ref{alg:lgd} in this paper.

For classification problems, efficient sampling from $\mcD_{y_v|X_v,w}$ is simpler because of the finite number of classes available. In fact, the discussion of Algorithm \ref{alg:lgd_class} and Theorem \ref{thm:lgd_conv_class} is fairly general, and does not use any properties of $\smax(.)$ other than piecewise polynomial implementation and 1-Lipschitzness. For a loss function different from the one assumed in Theorem \ref{thm:lgd_conv_class}, we can simply replace the function $\smax(.)$ by any function that is 1-Lipschitz and the guarantees still hold true. It is common to assume $Pr(y=j|x,w) = \frac{\exp(f(x;w)[j])}{\sum_{j'} \exp(f(x;w)[j'])}$ for classification problems, which does not satisfy the piecewise polynomial assumption made previously. However, as noted previously, hardware implementations of $\exp(.)$ are always piecewise-polynomial. Thus, the reader suffers from an approximation error in implementing their assumption on hardware, but none from our proposed algorithm.

\section{Generalization Guarantees}\label{sec:guarantees}
In this Section we show generalization guarantees on $\phi_{ERM}$, which are the combined set of hyperparameters $\{\theta,\sqrt{\eta}\}$ that minimize the meta learning loss as we note in Section \ref{sec:setting}. While there are several empirical approaches to find the minimizer of the meta learning loss as mentioned in Section \ref{sec:rw} such as the hypergradient or metagradient methods \citep{hyperparam_grad_dougal15, engstrom2025optimizingmltrainingmetagradient}, we assume access to the ERM solution directly. We present some empirical evidence using gradient-based methods for finding ERM hyperparameters in Section \ref{sec:expts}. In this Section, we first discuss bounds on generalization for regression and then show analogous results for classification.

\subsection{Hyperparameter Generalization for Regression}

The following Theorem establishes the Pseudo-dimension of the validation loss function class for regression over all hyperparameters $\theta$ and learning rates $\eta$, using the Goldberg-Jerrum framework from Section \ref{sec:bckg}. The validation loss function class is denoted as $\mathcal{L}_v$ below, which is the set of functions identified by the hyperparameters $\theta, \sqrt{\eta}$. that take as input the problem instance (comprising of training and validation data for a task) and the Langevin noise to be used during LGD, and outputs the obtained validation loss on this problem instance.

\begin{theorem}\label{thm:pdim_comp}
    In Algorithm \ref{alg:lgd},
    assume that all entries of $r(.;.)$ can be computed using a GJ algorithm with degree $\Delta_R$ and predicate complexity $\Lambda_R$ in its inputs $w_{(k)}$ and $\theta$. Similarly, assume the regression function $f(.;.)$ is differentiable and can be computed using a GJ algorithm with degree $\Delta_F$ and predicate complexity $\Lambda_F$ in its inputs. Finally, assume that the loss function $l(.,.)$ can be computed using a GJ algorithm with degree $\Delta_L$ and predicate complexity $\Lambda_L$.
    Define the function class 
    $\mcL_v = \left\{l_v(\{\theta, \sqrt{\eta}\},., .):\R^{(B+b)\times d}\times \mathcal{P}\rightarrow \R| \theta\in\R^{\dhp}\right\}$, 
    as the class of validation loss functions for given hyperparameter values for Algorithm \ref{alg:lgd}. The pseudo-dimension of $\mcL_v$ satisfies:
    \begin{align*}
        Pdim(\mcL_v) =  O\left((B+b)\dhp\log(\Delta_R + \Delta_F\Delta_L) + \dhp log((B+b)((n+n_v)(\Lambda_F + \Lambda_L) + \Lambda_R))\right).
    \end{align*}
\end{theorem}
\begin{proof}
    In the following, we give a GJ algorithm corresponding to the function class $\mcL_v$. For a task $P\in\mathcal{P}$ and Langevin noise $\xi\in\R^{(B+b)\times d}$, the inputs to the GJ algorithm $\Gamma_{(\xi, P),z}$ are 
    $\{\theta, \sqrt{\eta}\}$, 
    and we want it to compute whether 
    $l_v(\{\theta, \sqrt{\eta}\}, \xi,P)\ge z$ for $z\in\R$. 
    To do so, we first compute the predicted $\hw^t$ for each task $t$ following the GJ algorithms for computing $f$ and $r$ respectively and Algorithm \ref{alg:lgd} which only performs arithmetic operations on these values. The update step for $w_{(k)}$ can be analyzed as follows:
    \begin{enumerate}
        \item \textbf{$r(w_{(k)};\theta)$}: Predicate complexity $\Lambda_R$ and degree $\Delta_R$ in the inputs $\theta$ and in $w_{(k)}$.
        \item \textbf{$\nabla_{y'}l(y,y'= f(X;w_{(k)}))$}: We compute this as the derivative of the final rational functions in the computation of $l(y,y'= f(X;w_{(k)}))$. Thus, for each element in the inputs $y$ and $f(X;w_{(k)})$, we would get a degree of $2\Delta_L$ and predicate complexity of $\Lambda_L$. Over $n$ entries, we get a degree of $2\Delta_L$ and predicate complexity of $n\Lambda_L$. 
        The input $f(X;w_{(k)})$ has degree $\Delta_F$ and predicate complexity $\Lambda_F$ in each of its inputs from given assumptions, resulting in an overall degree $\Delta_F$ and predicate complexity $n\Lambda_F$. Thus, this term has an overall degree of $2\Delta_F\Delta_L$ and predicate complexity $n\Lambda_L+n\Lambda_F$ in the inputs $X,y,w_{(k)}$.
        \item \textbf{$\nabla_w f(X;w_{(k)})$}: We compute this as the derivative of the final rational functions in the computation of $f(X;w_{(k)})$.
        Thus for each entry in $X$, it has predicate complexity $\Lambda_F$ and degree at most $2\Delta_F$. Over all entries in the input, we have predicate complexity of $n\Lambda_F$ and degree $2\Delta_F$.
        \item \textbf{$g_{(k)}$}: Combining the above steps, we get that this term has a degree of $2\Delta_F\Delta_L + 2\Delta_F + \Delta_R$ and a predicate complexity of $n(2\Lambda_F+\Lambda_L) + \Lambda_R$.
        \item \textbf{$w_{(k+1)}$}: Since $\eta$ has degree 2 in $\sqrt{\eta}$, this term has a degree of $2\Delta_F\Delta_L + 2\Delta_F + \Delta_R + 2$ and a predicate complexity of $n(2\Lambda_F+\Lambda_L) + \Lambda_R$ in terms of $w_{(k)}$, $\theta$ and $\sqrt{\eta}$.
    \end{enumerate}

    Now, we are interested in computing the degree and predicate complexity of $w_{(k + 1)}$ in terms of $\theta, \sqrt{\eta}$,
    $w_{(0)}$, 
    instead of $w_{(k)}$. Solving for this recurrence relation we can conclude that $w_{(k + 1)}$ has a degree of $(2\Delta_F(\Delta_L+1)+\Delta_R+2)^{(k+1)}$
    in the inputs $\theta, \sqrt{\eta}$.
    Since each step of the algorithm uses $n(2\Lambda_F+\Lambda_L) + \Lambda_R$
    rationals in conditional statements, the overall algorithm will need $(k+1)(n(2\Lambda_F+\Lambda_L) + \Lambda_R)$
    rationals in conditional statements. 

    We can thus conclude that $f(X_v;w_{(k)})$ has a predicate complexity of $k(n(2\Lambda_F+\Lambda_L) + \Lambda_R) + n_v\Lambda_F$ and degree $\Delta_F(2\Delta_F(\Delta_L+1)+\Delta_R+2)^{k}$. 
    Thus computing $\hat{f}_v$ has a predicate complexity of  $(B + b)(n(2\Lambda_F+\Lambda_L) + \Lambda_R) + n_v\Lambda_F$ and a degree of $\Delta_F(2\Delta_F(\Delta_L+1)+\Delta_R+2)^{B+b}$.
    Finally, the validation loss has degree $\Delta_L$ and predicate complexity $\Lambda_L$ for each entry of $\hat{f}_v$. Thus it has a degree of $\Delta_F\Delta_L(2\Delta_F(\Delta_L+1)+\Delta_R+2)^{B+b}$ in the algorithm's inputs and a predicate complexity of $(B + b)(n(2\Lambda_F+\Lambda_L) + \Lambda_R) + n_v\Lambda_F + n_v\Lambda_L$. 
    Computing $l_v(\{\theta, \sqrt{\eta}\}, \xi,P)\ge r$
    takes one more conditional query. We can thus conclude from Lemma \ref{lem:gj_pdim} that the pseudo-dimension is given as follows:
    \begin{align*}
        Pdim(\mcL_v) &= O(\dhp\log\left(\left[\Delta_F\Delta_L(2\Delta_F(\Delta_L+1)+\Delta_R+2)^{B+b}\right] \left[(B + b)(n(2\Lambda_F+\Lambda_L) + \Lambda_R) + n_v\Lambda_F + n_v\Lambda_L \right] \right)\\
        &= O\left((B+b)\dhp\log(\Delta_R + \Delta_F\Delta_L) + \dhp log((B+b)((n+n_v)(\Lambda_F + \Lambda_L) + \Lambda_R)\right).  
    \end{align*}
\end{proof}

Prior work on learning hyperparameters for linear regression using the elastic net \citep{balcan_23} give a pseudo-dimension upper bound of $O(d)$, where they consider perfect optimization and only $h=2$ hyperparameters. As we see in Theorem \ref{thm:erm_bayes_conv}, we obtain pseudo-dimension bounds of $O(dh\log(d(n+n_v))$ based on typical values of $B,b$, without assuming perfect optimization for the much more general case of regression with a convex loss. That is, we match the bounds of prior work, up to logarithmic terms, for a much more general setting while accounting for optimization errors as well.

In order to show our generalization guarantees, we make the following assumption on the boundedness of the loss function $l(y_p,y_t)$ similar to \citet{balcan_23, goyal_25}, valid over all $(X,y,X_v,y_v)\in\mathcal{P}$ and $x_v^{(i)},y_v^{(i)}\in X_v,y_v$, and for all possible estimators $\hat{f}_v$: 
\begin{assumption}\label{assum:bound}
    $l(\hat{f}_v(x_v^{(i)},X,y),y_v^{(i)})\in[0,C]$.
\end{assumption}

Below we give generalization guarantees on $\phi_{ERM}(\xi)$.

\begin{theorem}\label{thm:gen_guarantee}
Assume there exist GJ algorithms to compute the functions $l(\cdot,\cdot)$, $r(\cdot;\cdot)$,$f(\cdot;\cdot)$ with a dimension-independent degree and predicate complexity. Given Assumption \ref{assum:bound} holds, $T = O\left(\frac{C^2}{\epsilon^2}\left(Pdim(\mcL_v)+\log(1/\delta)\right)\right)$ tasks are enough to ensure that with probability $\ge 1-\delta$,
    \begin{align*}
        l_v(\phi_{ERM}(\xi)) - l_v(\phi^*) &\le \epsilon.
    \end{align*}
    Here we denote the combined set of hyperparameters $\{\theta, \sqrt{\eta}\}$ by $\phi$.
\end{theorem}
\begin{proof}
We begin by defining the expected validation error using fixed randomness $\xi$ as follows, where $\xi^t$ and $S^t$ represent the noise used in Algorithm \ref{alg:lgd} of the $t$th task and the $t$th task of $S$ respectively, consistent with our notation of using superscripts to denote task numbers:
\begin{align*}
    l_v(\phi, \xi) &= \E_{S}\left[\frac{1}{T}\sum_t l_v(\phi, \xi^t, S^t)\right] = \frac{1}{T} \sum_t \E_P[l_v(\phi, \xi^t, P)]= \E_{S}[l_v(\phi, \xi, S)].
\end{align*}
The second equation above follows from the i.i.d. assumption of tasks within each problem set.
We can now decompose the target error term as follows:
    \begin{align*}
        l_v(\phi_{ERM}(\xi)) - l_v(\phi^*) &\le l_v(\phi_{ERM}(\xi)) - l_v(\phi_{ERM}(\xi),\xi,S) + l_v(\phi_{ERM}(\xi),\xi,S) - l_v(\phi^*(\xi),\xi,S)\\
        &\quad + l_v(\phi^*(\xi),\xi,S) - l_v(\phi^*(\xi), \xi)  + l_v(\phi^*(\xi),\xi) - l_v(\phi^*, \xi)  + l_v(\phi^*, \xi) -l_v(\phi^*).
    \end{align*}
    Since by definition $l_v(\phi^*(\xi),\xi) - l_v(\phi^*, \xi) \le 0$ and $l_v(\phi_{ERM}(\xi),\xi,S) - l_v(\phi^*(\xi),\xi,S) \le 0$, we can simplify the above expression as follows:
    \begin{align}
            l_v(\phi_{ERM}(\xi)) - l_v(\phi^*) 
            &\le l_v(\phi_{ERM}(\xi)) - l_v(\phi_{ERM}(\xi),\xi,S)  + l_v(\phi^*(\xi),\xi,S) - l_v(\phi^*(\xi),\xi)\nonumber\\
            &\quad + l_v(\phi^*, \xi) -l_v(\phi^*).\label{eq:err_decomp}
        \end{align}
    
    \noindent We bound the first term using standard pseudo-dimension theory. Notably, $l_v(\phi_{ERM}(\xi)) - l_v(\phi_{ERM}(\xi),\xi,S) \le \sup_\phi |l_v(\phi) - l_v(\phi,\xi,S)| \le \epsilon$ for $T = O\left(\frac{C^2}{\epsilon^2}\left(Pdim(\mcL_v)+\log(1/\delta)\right)\right)$ using uniform convergence bounds from Theorem \ref{thm:pdim_generalization}. 

    For the second term in Equation \ref{eq:err_decomp}, we use Assumption \ref{assum:bound} and Hoeffding's inequality \ref{thm:Hoeffding} to note that $T = O\left(\frac{C^2}{2\epsilon^2}\log(1/\delta)\right)$ are enough to ensure with probability $\ge 1-\delta$ over the sampling of tasks in $S$, that $l_v(\phi^*(\xi),\xi,S) - l_v(\phi^*(\xi),\xi) \le \epsilon$.
    Similarly for the last term, the same number of tasks are enough to ensure with probability $\ge 1-\delta$ over the sampling of noise $\xi$ for all tasks, that $l_v(\phi^*, \xi) -l_v(\phi^*) \le \epsilon$.
    We recover the desired bounds by combining the above inequalities and replacing $\delta$ with $\delta/3$ and $\epsilon$ with $\epsilon/3$ and absorbing constants into the $O(.)$ notation.
\end{proof}

\begin{remark}
We note that the above bounds can be tightened up to constant factors by accounting for the variance of algorithm \ref{alg:lgd} which decreases with $B,b$ as seen from Proposition \ref{prop:bounds}. Denote $Var_\xi(\hat{f}_v(x_v,X,y)) = V$ as the variance of predictions of any validation input $x_v$. Assume that the loss function $l(.,.)$ is Lipschitz in the first term with Lipschitz constant $L_l$ so that the variance of the loss function  $Var_\xi(l(\hat{f}_v(x_v,X,y),y_v)) = L_l^2V$. Thus, $Var_\xi(l_v(\theta,\xi,P)) \le L_l^2V$, for all $\theta$ and $P$ since $l_v(\theta,\xi,P)$ is an average over $n_v$ validation examples, which are not necessarily i.i.d. Thus $T=O\left(\frac{L_l^2(V+C\epsilon)}{\epsilon^2}\log(1/\delta)\right)$ are enough to ensure $l_v(\phi^*, \xi) -l_v(\phi^*) \le \epsilon$ using Bernstein's inequality (Theorem \ref{thm:Bernstein}).
It may be possible to further tighten this analysis to get a $O(1/(n_v\epsilon^2))$ bound on $T$ by assuming i.i.d. for the validation examples and keeping track of the correlation induced by a common noise vector $\xi$ and common training data between predictions of different validation inputs. However since the dominant terms on $T$ in the bound of Theorem \ref{thm:gen_guarantee} remain $O(C^2/\epsilon^2)$, we defer such a discussion to future work.
\end{remark}

We are now able to combine the bounds of Theorems \ref{thm:lgd_conv} and \ref{thm:gen_guarantee} to give a bound on the convergence of the ERM estimator to the Bayes optimal estimator as given below. 

\begin{theorem}\label{thm:erm_bayes_conv}
    Assume the setting of Theorem \ref{thm:lgd_conv} where we assume Assumptions \ref{ass:loss_lips}-\ref{assum:bound} hold and a well-specified regression problem such that there exists a $\ws$ for which $y \sim \mcN(f(X;\ws), I_n)$ and $y_v \sim \mcN(f(X_v;\ws),I_{n_v})$. That is, there exists a ground truth $\ws$ for which the problem is exactly solvable up to Gaussian noise. Further assume that the ground truth $\ws \sim \mcD_W$, we use the squared loss, $l(y,y') = \|y-y'\|^2$, and there exist GJ algorithms to compute the functions $r(\cdot;\cdot)$,$f(\cdot;\cdot)$ with a dimension-independent degree and predicate complexity. Denote $\zeta = 2\sqrt{C}\frac{\|f\|_{Lip}}{\epsilon_2}$. If we learn ERM hyperparameters of Algorithm \ref{alg:lgd} with $B$ within $O\left(\dw\zeta^2\log\left((\dw+\|w_{(0)}-w_{min}\|^2) \zeta^2\right)\right)$ and $b$ within $O\left(\dw\log(2n_v/\delta)\zeta^4\right)$,
    then
    \begin{align*}
        T=O\left(\frac{C^2}{\epsilon_1^2}dh\log(d(n+n_v))\left(\zeta^2\log\left(\zeta\right) +  \zeta^4\log(2/\delta)\right)\right)
    \end{align*}
    tasks are enough to ensure that, with probability $\ge 1-\delta$,
    \begin{align*}
        B_\pi(\hat{f}_{ERM,\xi}) - B_\pi(\hat{f}_v^{opt}) \le \epsilon_1 + \sqrt{n_v}\epsilon_2.
    \end{align*}
\end{theorem}
\begin{proof}
    Similar to Theorem \ref{thm:lgd_conv}, define $\theta^p$ such that $r(w;\theta^p) = \nabla_w \log \mathcal{D}_W(w)$ and $\eta^{p}$ such that $\Omega\left(\frac{\zeta^2\log(1/\delta)}{b}\right) \le \eta^{p}\le O\left(\frac{1}{\zeta^2d}\right)$. We decompose the target bound as follows:
    \begin{align*}
        B_\pi(\hat{f}_{ERM,\xi}) - B_\pi(\hat{f}_v^{opt}) &= l_v(\phi_{ERM}(\xi)) - B_\pi(\hat{f}_v^{opt})\\
        &=  l_v(\phi_{ERM}(\xi)) - l_v(\phi^*) + l_v(\phi^*) - l_v(\{\theta^p,\eta^{p}\}) + l_v(\{\theta^p,\eta^{p}\}) - B_\pi(\hat{f}_v^{opt}).
    \end{align*}
    Note by definition that $l_v(\phi^*) - l_v(\{\theta^p,\eta^{p}\}) \le 0$, since $\phi^*$ is the best hyperparameter configuration for $\phi$ and $\eta$. In order to bound $l_v(\{\theta^p,\eta^{(p)}\}) - B_\pi(\hat{f}_v^{opt})$, we simplify the expression as follows:
    \begin{align*}
        l_v(\{\theta^p,\eta^{(p)}\}) - B_\pi(\hat{f}_v^{opt}) &= \E_{X,y,X_v}[ \E_{y_v|X_v,X,y}[\|y_v - \hat{f}_v^p\|^2 - \|y_v - \hat{f}_v^{opt}\|^2]]\\
        &= \E_{X,y,X_v}[\E_{y_v|X_v,X,y}[\|\hat{f}_v^p - \hat{f}_v^{opt}\|(\|y_v - \hat{f}_v^p\| + \|y_v - \hat{f}_v^{opt}\|) ]]\\
        &\le 2\sqrt{C}\|\hat{f}_v^p - \hat{f}_v^{opt}\|. \quad \text{(Assumption \ref{assum:bound} implies $\|y_v - \hat{f}_v\|\le\sqrt{C}$)}
    \end{align*}
    It is easy to see from Theorem \ref{thm:lgd_conv} that $l_v(\{\theta^p,\eta^{(p)}\}) - B_\pi(\hat{f}_v^{opt}) \le \sqrt{n_v}\epsilon_2$  with probability $\ge 1-n_v\delta$ for $\eta^p = O(1/(\zeta^2\dw))$,$b= O\left( \frac{\zeta^2\log(1/\delta)}{\eta^{p}}\right)$ and $B = O\left(\dw\zeta^2\log\left((\dw+\|w_{(0)}-w_{min}\|^2)\zeta^2\right)\right)$. Note that the constant $2\sqrt{C}$ is absorbed into the definition of $\zeta$ above.
    We substitute our choice of $\eta^p$ in the bound for $b$ and replace $\delta$ with $\delta/2n_v$ to obtain $b = O\left(\dw\zeta^4\log(2n_v/\delta)\right)$ . Thus, with probability $\ge 1-\delta/2$, $l_v(\{\theta^p,\eta\}) - B_\pi(\hat{f}_v^{opt}) \le \sqrt{n_v}\epsilon_2$. 
    
    Further, we know from Theorem \ref{thm:gen_guarantee} that $l_v(\phi_{ERM}(\xi)) - l_v(\phi^*) \le \epsilon_1$ with probability $\ge 1-\delta$ for a choice of $T=O\left(\frac{C^2}{\epsilon_1^2}\left(Pdim(\mathcal{L}_v) + \log(3/\delta)\right)\right)$. For the squared loss, we note that $\Delta_L = 2$ and $\Lambda_L =0$, which are both dimension-independent.
    We  instantiate the pseudo-dimension bounds from Theorem \ref{thm:pdim_comp}, as follows:
    \begin{align*}
        Pdim(\mcL_v) &= O((B+b)h + h\log((B+b)(n+n_v))\\
        &= O\left((d\zeta^2\log(d\zeta^2) + d\zeta^4\log(2n_v/\delta))h + h\log((d\zeta^2\log(d\zeta^2) + d\zeta^4\log(2n_v/\delta))(n+n_v))\right)\\
        &= O\left((d\zeta^2\log(d\zeta^2) + d\zeta^4\log(2n_v/\delta))h + h(\log(d(n+n_v))+\log(\zeta^2+\zeta^4\log(2n_v/\delta)))\right)\\
        &= O\left(d\log(d(n+n_v))(\zeta^2\log(\zeta^2) + \zeta^4\log(2/\delta))h + h(\log(d(n+n_v))+\log(\zeta^2+\zeta^4\log(2/\delta)))\right)\\
        &= O\left(dh\log(d(n+n_v))(\zeta^2\log(\zeta) + \zeta^4\log(2/\delta))\right).
    \end{align*}

    Thus, we have $l_v(\phi_{ERM}(\xi)) - l_v(\phi^*) \le \epsilon_1$ with probability $\ge 1-\delta/2$ for a choice of $T$ that satisfies:
    \begin{align*}
        T &= O\left(\frac{C^2}{\epsilon_1^2}\left(dh\log(d(n+n_v))\left(\zeta^2\log\left(\zeta\right) +  \zeta^4\log(2/\delta)\right) + \log(2/\delta)\right)\right)\\
        &= O\left(\frac{C^2}{\epsilon_1^2}\left(dh\log(d(n+n_v))\left(\zeta^2\log\left(\zeta\right) +  \zeta^4\log(2/\delta)\right) \right)\right).
    \end{align*}

    We thus recover the target bounds by combining both inequalities in the decomposition.
\end{proof}

We observe that the bound in the above theorem achieves a conservative bound on $T$ of the form $O(1/\epsilon^6)$ if $\epsilon_1 = \epsilon_2 = \epsilon$, while typical generalization error bounds (for example, \citep{balcan_23, goyal_25}) only suffer from a quadratic rate of $O(1/\epsilon^2)$. We note however that prior work assumes perfect optimization of the loss functions, and fail to achieve any concentration to the Bayes' optimal if the prior density is not Gaussian, Laplace or a product of the two. 

\subsection{Hyperparameter Generalization for Classification}
Similar to the results in the above section, we can give results on generalization of hyperparameters for classification. The following theorem gives a pseudo-dimension bound on the validation loss function class, again denoted as $\mathcal{L}_v$, which is the set of functions identified by the hyperparameters $\theta, \sqrt{\eta}$. that take as input the problem instance (comprising of training and validation data for a task) and the Langevin noise to be used during LGD, and outputs the obtained validation loss on this problem instance. Since loss functions for classification typically involve the $\log$ function, we assume that the gradient of the loss is computable by a piecewise-polynomial instead of the loss directly.

\begin{theorem}\label{thm:pdim_class}
    In Algorithm \ref{alg:lgd_class},
    assume that all entries of $r(.;.)$ can be computed using a GJ algorithm with degree $\Delta_R$ and predicate complexity $\Lambda_R$ in its inputs $w_{(k)}$ and $\theta$. Similarly, assume the regression function $f(.;.)$ is differentiable and all entries can be computed together using a GJ algorithm with degree $\Delta_F$ and predicate complexity $\Lambda_F$ in its inputs. Further assume that all entries of the function $\smax$ can be computed using a GJ algorithm of degree $\Delta_{SM}$ and predicate complexity $\Lambda_{SM}$.
    Finally, assume that all entries of the gradient of the loss function $l(.,.)$ can be computed using a GJ algorithm with degree $\Delta_{L'}$ and predicate complexity $\Lambda_{L'}$, and given $\smax(\hat{f}_v)$ for all $n_v$ validation entries, computing whether $l_v(\{\theta, \sqrt{\eta}\}, \xi,P)\ge r$ requires a GJ algorithm of degree $n_v\Delta_{Lv}$ and predicate complexity $(n_v\Lambda_{Lv} + 1)$ in the entries of $\smax(\hat{f}_v)$ for any $r$.
    Define the function class 
    $\mcL_v = \left\{l_v(\{\theta, \sqrt{\eta}\},., .):\R^{(B+b)\times d}\times \mathcal{P}\rightarrow \R| \theta\in\R^{\dhp}\right\}$, 
    as the class of validation loss functions for given hyperparameter values for Algorithm \ref{alg:lgd_class}. The pseudo-dimension of $\mcL_v$ satisfies:
    \begin{align*}
        Pdim(\mcL_v) &=  O\left((B+b)\dhp\log(n_v(\Delta_R + \Delta_F\Delta_{L'}\Delta_{Lv}\Delta_{SM}))\right.\\
        &\quad \quad \quad \quad \left.+ \dhp log((B+b)((n+n_v)(\Lambda_F + \Lambda_{L'} + \Lambda_{Lv} + \Lambda_{SM}) + \Lambda_R)\right).  
    \end{align*}
\end{theorem}
\begin{proof}
    While Algorithm \ref{alg:lgd} performs a scalar-vector multiplication in Line 3, Algorithm \ref{alg:lgd_class} performs a matrix-vector multiplication. This however does not change the degree or the predicate complexity of the computation, and we can proceed similar to the proof of Theorem \ref{thm:pdim_comp}. We thus note that computing all entries of $f(X_v;w_{(k)})$ has a predicate complexity of $k(n(2\Lambda_F+\Lambda_{L'}) + \Lambda_R) + n_v\Lambda_F$ and degree $\Delta_F(2\Delta_F(\Delta_{L'}+1)+\Delta_R+2)^{k}$. 
    
    Now, each softmax computation  uses $\Lambda_{SM}$ rationals in conditional statements and $\Delta_{SM}$ degree in its inputs. Thus the $k^{th}$ entry of the empirical average of $\hat{F}_v$ has a predicate complexity of $k(n(2\Lambda_F+\Lambda_{L'}) + \Lambda_R) + n_v(\Lambda_F + \Lambda_{SM})$ and degree $\Delta_{SM}\Delta_F(2\Delta_F(\Delta_{L'}+1)+\Delta_R+2)^{k}$. Thus $\hat{F}_v$ has a predicate complexity of  $(B + b)(n(2\Lambda_F+\Lambda_{L'}) + \Lambda_R) + n_v(\Lambda_F + \Lambda_{SM})$ and a degree of $\Delta_{SM}\Delta_F(2\Delta_F(\Delta_{L'}+1)+\Delta_R+2)^{B+b}$. 
    
    Finally, we need to compute if $l_v(\{\theta, \sqrt{\eta}\}, \xi,P)\ge r$. This has a predicate complexity of  $(B + b)(n(2\Lambda_F+\Lambda_{L'}) + \Lambda_R) + n_v(\Lambda_F + \Lambda_{SM} + \Lambda_{Lv}) + 1$ and a degree of $n_v\Delta_{Lv}\Delta_{SM}\Delta_F(2\Delta_F(\Delta_{L'}+1)+\Delta_R+2)^{B+b}$. We can thus conclude from Lemma \ref{lem:gj_pdim} that the pseudo-dimension is given as follows:
    \begin{align*}
        Pdim(\mcL_v) &= O(\dhp\log\left(\left[n_v\Delta_F\Delta_{Lv}\Delta_{SM}(2\Delta_F(\Delta_{L'}+1)+\Delta_R+2)^{B+b}\right]\times\right. \\
        &\quad \quad \quad \quad \left.\left[(B + b)(n(2\Lambda_F+\Lambda_{L'}) + \Lambda_R) + n_v(\Lambda_F + \Lambda_{Lv} + \Lambda_{SM}) + 1 \right] \right)\\
        &= O\left((B+b)\dhp\log(n_v(\Delta_R + \Delta_F\Delta_{L'}\Delta_{Lv}\Delta_{SM}))\right.\\
        &\quad \quad \quad \quad \left.+ \dhp log((B+b)((n+n_v)(\Lambda_F + \Lambda_{L'} + \Lambda_{Lv} + \Lambda_{SM}) + \Lambda_R)\right).  
    \end{align*}
\end{proof}

Prior work on learning hyperparameters for logistic regression on $m=2$ classes using $l1$ or $l2$ penalties \citep{balcan_23} give a pseudo-dimension upper bound of $O((n+n_v)^2+\log(1/\epsilon))$ for an $\epsilon$-approximate optimization algorithm, using $h=1$ hyperparameters. This is distinctly different from their results on regression which do not degrade with the number of samples, but do degrade with the dimensionality. This distinction heavily derives from the difference in analysis for both regression and classification done in \citet{balcan_23}, and the fact that they consider specialised algorithms that do not generalize beyond the linear case. The degradation in terms of the number of samples for classification arises from the analysis of the exact softmax function, the behavior of which can be much more irregular than the piecewise-polynomial approximation considered in this paper.
Our bounds on classification match our bounds on regression up to logarithmic terms owing to our consideration of general algorithms. In order to compare our bounds with those of \citet{balcan_23}, we first need to note that we do not consider averaging of our loss function across validation examples, so that Algorithm \ref{alg:lgd_class}, which is $mn_v\epsilon$-approximate as shown in Theorem \ref{thm:lgd_conv_class}, is $\epsilon$-approximate by the definition of \citet{balcan_23}. Further note that we do not have an algorithm that is always $\epsilon$-approximate, but is $\epsilon$-approximate with probability $\ge 1-\delta$.
Theorem \ref{thm:erm_bayes_class} below shows a pseudo-dimension upper bound of the form $O\left(d\log(d(n+n_v))\log(n_v)\left(\frac{1}{\epsilon^2} \log(1/\epsilon) + \frac{1}{\epsilon^4}\log(2/\delta)\right)\right)$. We thus improve greatly in terms of the dependence on number of samples, which goes from quadratic to poly-logarithmic. As noted previously, our bounds use a piecewise-polynomial approximation of the softmax function, which exactly models the hardware implementation of the function, which leads to this improvement.
Our bounds however degrade in terms of the dimension, where we introduce a linear dependence, and the approximation error, which increases from logarithmic to polynomial. We credit both of these degradations to our consideration of a sampling algorithm (ULA), as opposed to the optimization algorithm of \citet{balcan_23} that does not generalize beyond specific regularizers. Such degradations are commonly seen when comparing between optimization and sampling algorithms. For example, convergence to an $\epsilon$-approximate solution of a strongly convex loss function requires $O(\log(1/\epsilon))$ steps \citep{nesterov2018lectures}, whereas sampling from a log-strongly-concave distribution requires polynomial in $1/\epsilon$ number of steps (Proposition \ref{lem:ula_conv}).

In addition to proving bounds on more general scenarios including hardware approximation errors, we also consider a more careful treatment of the bounds of \citet{balcan_23} that assume an $\epsilon$-approximate algorithm but assume perfect hardware implementation of the $\exp(.)$ function. We show that their pseudo-dimension upper bounds can be improved to $O(\log(n_v/\epsilon))$ in Appendix \ref{sec:rlr_pdim}, matching our logarithmic dependence on the number of samples. While we maintain our focus on bounds that follow hardware restrictions, we note that bounds analogous to ours can be derived assuming perfect implementation of the $\exp(.)$ function following the treatment of \citet{poly_approx_vitercik2020}, where the authors prove upper bounds on the Rademacher complexity of a class of algorithms that can be implemented approximately using piecewise polynomials.

We now show convergence of the ERM estimator to the Bayes optimal estimator for classification, similar to Theorem \ref{thm:erm_bayes_conv}.

\begin{theorem}\label{thm:erm_bayes_class}
    Assume the setting of Theorem \ref{thm:lgd_conv_class} where we assume Assumptions \ref{ass:e_approx}-\ref{assum:bound} hold and a well-specified classification problem such that $Pr(y^{(i)}=j|X^{(i)},w) = \smax(f(X^{(i)},w))[j]$
    and $Pr(y_v^{(i)}  = j|X_v^{(i)},w) = \smax(f(X_v^{(i)},w))[j]$. 
    Further assume that the ground truth $\ws \sim \mcD_W$.
    Assume the loss function $l(y=j,f) = -\log(\smax(f)[j])$ such that $l(y=j,f)$ is convex in all entries of $f$, and assume 
    there exist GJ algorithms to compute all entries of $\smax(.)$ and the functions  $r(\cdot;\cdot)$,$f(\cdot;\cdot)$ with a dimension-independent degree and predicate complexity. Denote $\zeta = \frac{\|f\|_{Lip}}{\epsilon_2}$. For some $a>0$ if Algorithm \ref{alg:lgd_class} outputs values such that $\hat{F}_v(X_v^{(i)})[j]\ge a $ 
    for all $i\in[n_v],j\in[m]$, and we learn ERM hyperparameters of Algorithm \ref{alg:lgd_class} with $B$ within $O\left(\dw\zeta^2\log\left((\dw+\|w_{(0)}-w_{min}\|^2) \zeta^2\right)\right)$ and $b$ within $O\left(\dw\log(2mn_v/\delta)\zeta^4\right)$,
    then
    \begin{align*}
        T=O\left(\frac{C^2}{\epsilon_1^2}dh\log(d(n+n_v))\log(n_v)\left(\zeta^2\log\left(\zeta\right) +  \zeta^4\log(2m/\delta)\right)\right)
    \end{align*}
    tasks are enough to ensure that, with probability $\ge 1-\delta$,
    \begin{align*}
        B_\pi(\hat{f}_{ERM,\xi}) - B_\pi(\hat{f}_v^{opt}) \le \epsilon_1 + mn_v\frac{\log(1/a)}{(1-a)}\epsilon_2.
    \end{align*}
\end{theorem}
\begin{proof}
    Similar to the proof of Theorem \ref{thm:erm_bayes_conv}, we decompose the target term as follows:
    \begin{align*}
        B_\pi(\hat{f}_{ERM,\xi}) - B_\pi(\hat{f}_v^{opt}) &= l_v(\phi_{ERM}(\xi)) - l_v(\phi^*) + l_v(\phi^*) - l_v(\{\theta^p,\eta^{p}\}) + l_v(\{\theta^p,\eta^{p}\}) - B_\pi(\hat{f}_v^{opt})\\
        &\le l_v(\phi_{ERM}(\xi)) - l_v(\phi^*) + l_v(\{\theta^p,\eta^{p}\}) - B_\pi(\hat{f}_v^{opt}).
    \end{align*}

    We expand the last term as follows:
    \begin{align*}
        l_v&(\{\theta^p,\eta^{p}\}) - B_\pi(\hat{f}_v^{opt}) = \E_{X,y,X_v}[\E_{y_v|X_v,X,y}[l(y_v, \hat{f}_v^p) - l(y_v, \hat{f}_v^{opt})]]\\
        &= \E_{X,y,X_v}\left[\sum_{i\in[n_v]}\sum_{j\in[m]} Pr(y_v^{(i)}=j|X_v,X,y)\right.\cdot\\
        &\qquad \qquad \qquad \qquad \qquad \qquad \left.\left(-\log\left(\smax(\hat{f}_v^{p}(X_v^{(i)}))[j]\right) + \log\left(\smax(\hat{f}_v^{opt}(X_v^{(i)}))[j]\right) \right)\right]\\
        &\le \E_{X,y,X_v}\left[\frac{\log(1/a)}{(1-a)} \sum_{i\in[n_v]}\sum_{j\in[m]} \left|\smax(\hat{f}_v^p(X_v^{(i)}))[j] - \smax(\hat{f}_v^{opt}(X_v^{(i)}))[j]\right|\right].
    \end{align*}
    To see the last step, first denote $\smax(\hat{f}_v^p(X_v^{(i)}))[j] = \hat{Pr}(y_v^{(i)}=j|X_v,X,y)$ and note that $Pr(y_v^{(i)}=j|X_v,X,y) = \smax(\hat{f}_v^{opt}(X_v^{(i)}))[j]$ from Theorem \ref{thm:lgd_conv_class}.
    Now if $\hat{Pr}(y_v^{(i)}=j|X_v,X,y) \ge Pr(y_v^{(i)}=j|X_v,X,y)$, the inequality follows immediately so assume $\hat{Pr}(y_v^{(i)}=j|X_v,X,y) < Pr(y_v^{(i)}=j|X_v,X,y)$. Since the function $\frac{-\log(r)}{(1-r)}$ is decreasing for $r<1$, we can write $\frac{-\log(r)}{(1-r)} \le \frac{\log(1/a)}{(1-a)}$ for $r>a$. The last step then follows by replacing $r$ with $\frac{\hat{Pr}(y_v^{(i)}=j|X_v,X,y)}{Pr(y_v^{(i)}=j|X_v,X,y)} \in [a,1) $. We thus have $l_v(\{\theta^p,\eta^{p}\}) - B_\pi(\hat{f}_v^{opt}) \le mn_v\frac{\log(1/a)}{(1-a)}\epsilon_2$ with probability $\ge 1-mn_v\delta$ from Theorem \ref{thm:lgd_conv_class} for $\eta^p = O(1/(\zeta^2\dw))$, and $b,B$ that satisfy $b= O\left( \frac{\zeta^2\log(1/\delta)}{\eta^{p}}\right)$ and $B = O\left(\dw\zeta^2\log\left((\dw+\|w_{(0)}-w_{min}\|^2)\zeta^2\right)\right)$. 
    We substitute our choice of $\eta^p$ in the bound for $b$ and replace $\delta$ with $\delta/2mn_v$ to obtain $b = O\left(\dw\zeta^4\log(2mn_v/\delta)\right)$ . Thus, with probability $\ge 1-\delta/2$, $l_v(\{\theta^p,\eta\}) - B_\pi(\hat{f}_v^{opt}) \le  mn_v\frac{\log(1/a)}{(1-a)}\epsilon_2$. 

    Lastly, we want to bound $l_v(\phi_{ERM}(\xi)) - l_v(\phi^*)$ using the pseudo-dimension of the function class $\mcL_v$, similar to the proof of Theorem \ref{thm:erm_bayes_conv}. For the given loss function, note that the gradient of $l$ is computed as the ratio of the gradient of the $\smax(.)$ function and the $\smax(.)$ function. Since we assume $\Delta_{SM}$ and $\Lambda_{SM}$ are both dimension-independent, we can conclude that both $\Delta_{L'}$ and $\Lambda_{L'}$ are dimension-independent as well. Further, note that computing whether $l_v(\{\theta, \sqrt{\eta}\}, \xi,P)\ge r$ for any $r$ given $\smax(\hat{f}_v)$ requires computing if the product of the appropriate softmax entries is greater or smaller than $\exp(r)$. Thus this computation has a degree of $n_v$ in the entries of $\smax(\hat{f}_v)$ and a predicate complexity of $1$. Thus, $\Delta_{Lv} = 1$ and $\Lambda_{Lv} = 0$. We can then instantiate the pseudo-dimension bounds from Theorem \ref{thm:pdim_class} as follows:
    \begin{align*}
        P&dim(\mcL_v) = O((B+b)h\log(n_v) + h\log((B+b)(n+n_v))\\
        &= O\left((d\zeta^2\log(d\zeta^2) + d\zeta^4\log(2mn_v/\delta))h\log(n_v) + h\log((d\zeta^2\log(d\zeta^2) + d\zeta^4\log(2mn_v/\delta))(n+n_v))\right)\\
        &= O\left((d\zeta^2\log(d\zeta^2) + d\zeta^4\log(2mn_v/\delta))h\log(n_v) + h(\log(d(n+n_v))+\log(\zeta^2+\zeta^4\log(2mn_v/\delta)))\right)\\
        &= O\left(d\log(d(n+n_v))\log(n_v)(\zeta^2\log(\zeta^2) + \zeta^4\log(2m/\delta))h + h(\log(d(n+n_v))+\log(\zeta^2+\zeta^4\log(2m/\delta)))\right)\\
        &= O\left(dh\log(d(n+n_v))\log(n_v)(\zeta^2\log(\zeta) + \zeta^4\log(2m/\delta))\right).
    \end{align*}

    Thus, using Theorem \ref{thm:gen_guarantee}, we have $l_v(\phi_{ERM}(\xi)) - l_v(\phi^*) \le \epsilon_1$ with probability $\ge 1-\delta/2$ for a choice of $T$ that satisfies:
    \begin{align*}
        T &= O\left(\frac{C^2}{\epsilon_1^2}\left(dh\log(d(n+n_v))\log(n_v)\left(\zeta^2\log\left(\zeta\right) +  \zeta^4\log(2m/\delta)\right) + \log(2/\delta)\right)\right)\\
        &= O\left(\frac{C^2}{\epsilon_1^2}\left(dh\log(d(n+n_v))\log(n_v)\left(\zeta^2\log\left(\zeta\right) +  \zeta^4\log(2m/\delta)\right) \right)\right).
    \end{align*}
    
    We recover the target bounds by combining both inequalities in the decomposition.
\end{proof}

\section{Experimental Results}\label{sec:expts}
In this Section, we present experimental results on the Langevin Gradient Descent (LGD) applied to few-shot learning of linear regression presented in Algorithm \ref{alg:lgd} and logistic regression presented in Algorithm \ref{alg:lgd_class}. We present results on three distinct priors with increasing complexity as detailed below.
For each of the three priors, we generate a synthetic dataset in a $d=10$ dimensional space consisting of $T=250$ tasks for both linear regression and logistic regression. For each task, we sample a ground truth $w^{*t}$ from the prior, and sample $n+n_v=500$ inputs from a Gaussian distribution with variance $10$, $X^t\sim N(0,10\mathbb{I}_{10})^{\otimes (n+n_v)}$. 

For our experiments on linear regression, we sample standard normal noise $\epsilon^t\sim N(0,1)^{\otimes (n+n_v)}$ and set $y^t = X^tw^{*t} + \epsilon^t$. We measure loss using the squared loss, $l(y_t,y_p) = (y_p-y_t)^2 $ to match the setting of Theorem \ref{thm:lgd_conv} under which the optimal LGD algorithm is also Bayes' optimal as long as the prior is log-strongly-concave. Figure \ref{fig:lgd_plots} shows the average validation loss on the held-out set of $200$ tasks against increasing number of training samples for the theoretically optimal algorithms against the meta-learned algorithm trained on $50$ tasks. 

For our experiments on logistic regression, we define the softmax using python libraries as our $\smax$ function, since the internal implementations as piecewise polynomial as desired. We sample class labels $0$ or $1$ for each training or validation sample using $[0,Xw]$ as our logits. That is, for the $i^{th}$ sample, $Pr(y^{t(i)} = 1|X^{t(i)},w^{*t}) = \smax([0,X^{t(i)}w^{*t}])[1] $. We measure the loss using the log of $\smax$, $l(y_t=j,f) = -\log(\smax(f)[j]) $ to match the setting of Theorem \ref{thm:lgd_conv_class} under which the optimal LGD algorithm is also Bayes' optimal as long as the prior is log-strongly-concave. Figure \ref{fig:lgd_class_plots} shows the average validation loss on the held-out set of $200$ tasks against increasing number of training samples for the theoretically optimal algorithms against the meta-learned algorithm trained on $50$ tasks.

We note that our experiments on the isotropic Gaussian prior match linear or logistic regression with a tunable l2 penalty, settings covered in the analysis of prior work like \citet{enet_2022, balcan_23, goyal_25}. While our settings are the same for the isotropic Gaussian prior, prior work does not give empirical evidence of the success of their algorithm for the special cases considered therein. Further note that prior work does not give any guarantees for the more general priors considered in this work, namely the non-isotropic Gaussian and softplus scaled Gaussian priors. We describe the priors and optimal algorithms in detail below.

\begin{figure}[htbp]
    \centering
    
    \begin{subfigure}[t]{0.3\textwidth}
        \centering
        \includegraphics[width=\linewidth]{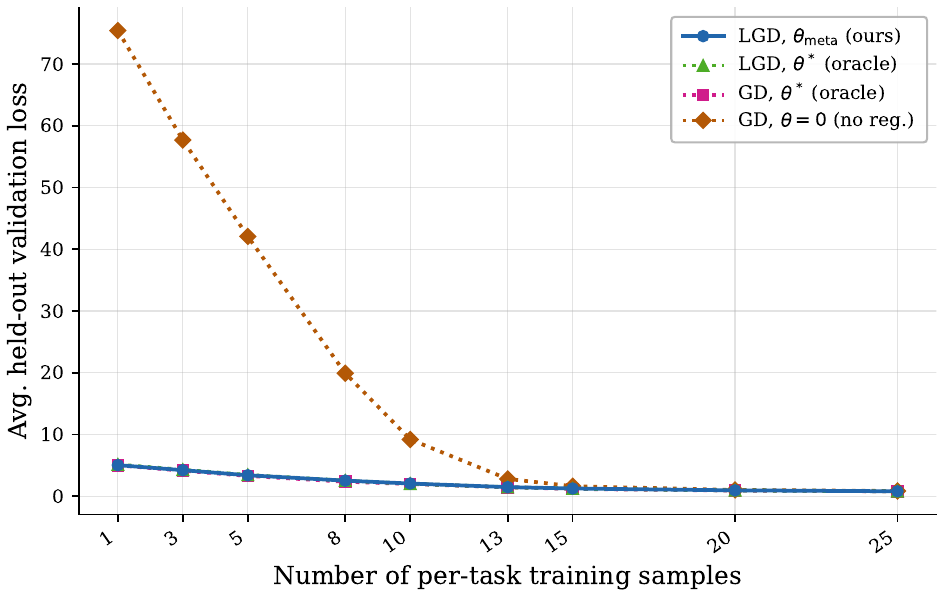}
        \caption{Isotropic Gaussian prior}
    \end{subfigure}
    \hfill
    \begin{subfigure}[t]{0.3\textwidth}
        \centering
        \includegraphics[width=\linewidth]{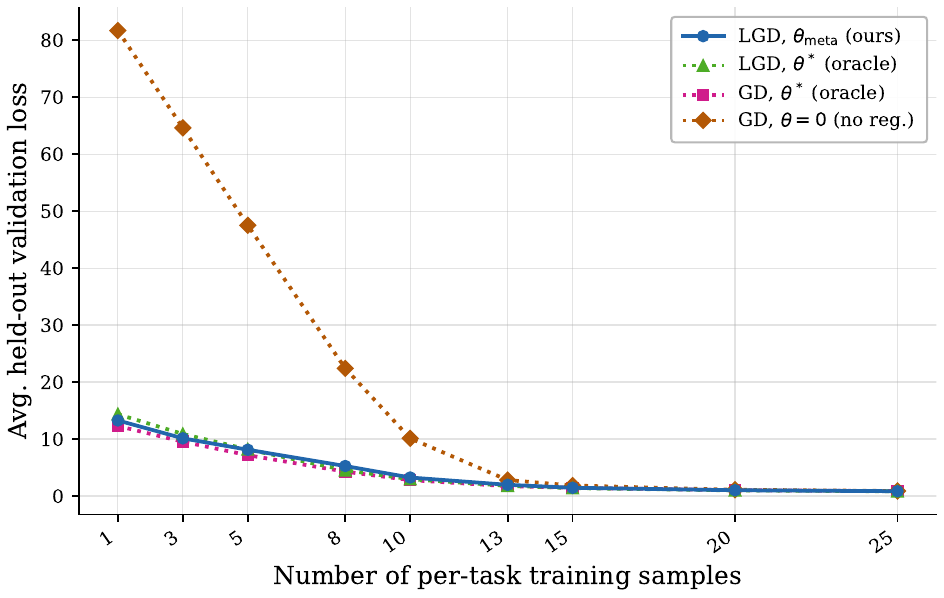}
        \caption{Non-isotropic Gaussian prior}
    \end{subfigure}
    \hfill
    \begin{subfigure}[t]{0.3\textwidth}
        \centering
        \includegraphics[width=\linewidth]{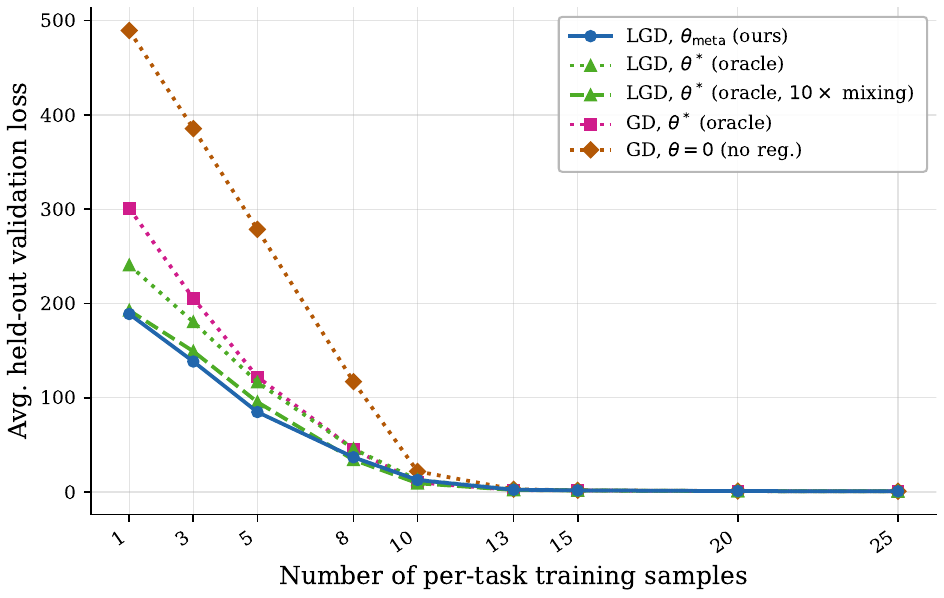}
        \caption{Gaussian with Softplus scaling}
    \end{subfigure}
    
    \caption{Comparison of oracle gradient descent (GD) and Langevin Gradient Descent (LGD) algorithms against meta-learned LGD algorithms for few-shot linear regression. We synthetically generate 250 linear regression tasks in a $d=10$ dimensional space by sampling ground truth vectors from log-strongly-concave priors, inputs $X^t \sim N(0,10\mathbb{I}_{10})^{\otimes (n+n_v)}$ and noise $\epsilon^t \sim N(0,1)^{\otimes n+n_v}$ and set $y^t = X^tw^{*t} + \epsilon^t$. We meta-learn using the first 50 tasks and evaluate on the last 200 for all algorithms. We present the average validation loss on held-out tasks against increasing number of training samples per task.
    (a) Shows comparison on an isotropic Gaussian prior on $w^*$ where the oracle LGD and GD algorithms perform comparatively and outperform GD without regularization. The meta-learned LGD algorithm matches the oracle baselines. (b) Shows comparison on a non-isotropic Gaussian prior, where the LGD oracle performs slightly worse than the GD oracle due to noise, but both out-perform plain GD with no regularization. The meta-learned LGD performs almost on par with the oracle LGD. (c) Shows comparison on a Gaussian prior scaled with Softplus, resulting in a prior with mean$\ne$mode. We see that the LGD oracle beats the GD oracle due to the asymmetric nature of the prior, and the GD oracle beats GD without regularization. The meta-learned LGD algorithm beats both of these oracles, and performs on par or better than the LGD oracle with a $10\times$ higher number of iterations.}
    \label{fig:lgd_plots}
\end{figure}

\begin{figure}[htbp]
    \centering
    
    \begin{subfigure}[t]{0.3\textwidth}
        \centering
        \includegraphics[width=\linewidth]{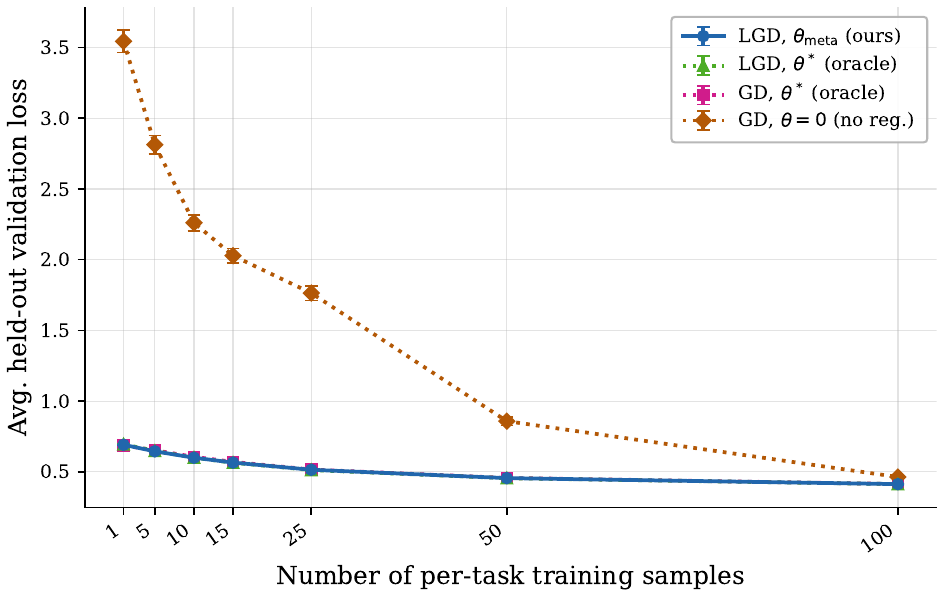}
        \caption{Isotropic Gaussian prior}
    \end{subfigure}
    \hfill
    \begin{subfigure}[t]{0.3\textwidth}
        \centering
        \includegraphics[width=\linewidth]{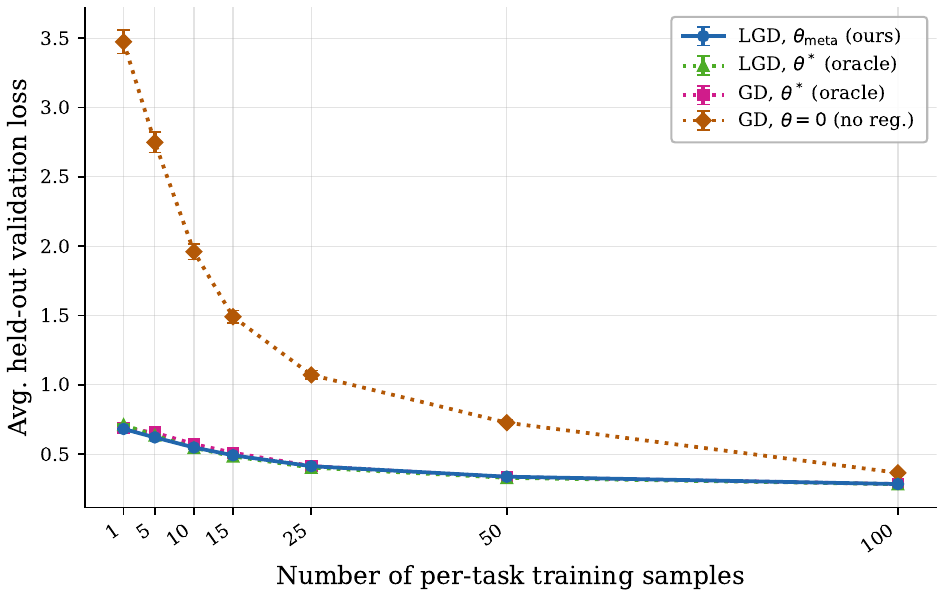}
        \caption{Non-isotropic Gaussian prior}
    \end{subfigure}
    \hfill
    \begin{subfigure}[t]{0.3\textwidth}
        \centering
        \includegraphics[width=\linewidth]{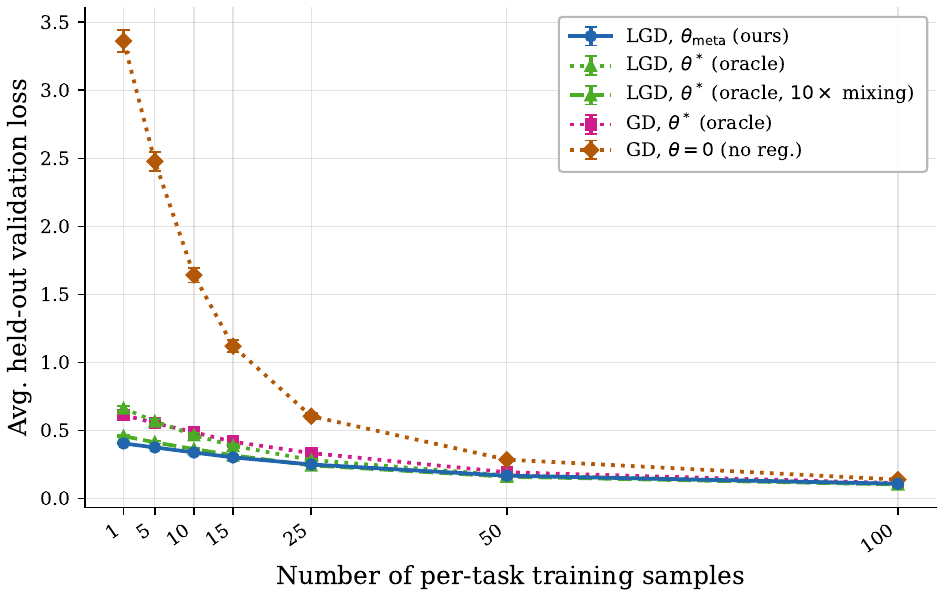}
        \caption{Gaussian with Softplus scaling}
    \end{subfigure}
    
    \caption{Comparison of oracle gradient descent (GD) and Langevin Gradient Descent (LGD) algorithms against meta-learned LGD algorithms for few-shot logistic regression. We synthetically generate 250 logistic regression tasks in a $d=10$ dimensional space by sampling ground truth vectors from log-strongly-concave priors, inputs $X^t \sim N(0,10\mathbb{I}_{10})^{\otimes (n+n_v)}$ and sample output labels that satisfy for each example $i$, $Pr(y^{t(i)} = 1|X^{t(i)},w^{*t}) = \smax([0,X^{t(i)}w^{*t}])[1] $. We meta-learn using the first 50 tasks and evaluate on the last 200 for all algorithms. We present the average validation loss on held-out tasks against increasing number of training samples per task.
    (a) Shows comparison on an isotropic Gaussian prior on $w^*$ where the oracle LGD and GD algorithms perform comparatively and outperform GD without regularization. The meta-learned LGD algorithm matches the oracle baselines. (b) Shows comparison on a non-isotropic Gaussian prior, where the LGD oracle performs slightly better than the GD oracle, because the GD solution is no longer Bayes' optimal. Both out-perform plain GD with no regularization. The meta-learned LGD performs on par with the oracle LGD. (c) Shows comparison on a Gaussian prior scaled with Softplus, resulting in a prior with mean$\ne$mode. We see that the LGD oracle beats the GD oracle due to the asymmetric nature of the prior, and the GD oracle beats GD without regularization. The meta-learned LGD algorithm beats both of these oracles, and performs on par or better than the LGD oracle with a $10\times$ higher number of iterations.}
    \label{fig:lgd_class_plots}
\end{figure}


\subsection{Isotropic Gaussian Prior}
Under the isotropic Gaussian prior, we sample each entry of $w^{*t}$ from a Gaussian distribution of variance $0.1$ so that $w^{*t}\sim N(0,0.1\mathbb{I}_{10}) $. Under this prior, Theorem \ref{thm:lgd_conv} gives that the Bayes' optimal estimator for linear regression is the mean of the posterior distributed as $Pr(w|X,y)\propto \exp{\left(-\frac{\|y-Xw\|^2}{2} - \frac{\|w\|^2}{2\times 0.1}\right)}$. Since the mean of this posterior is equal to its mode, the optimal Bayes' estimator is equal to both the LGD estimator and GD estimator using an l2 regularizer $\theta\|w\|^2$ with $\theta^*=10$. For logistic regression the posterior of $w$ is not a Gaussian and the LGD estimator is Bayes' optimal though the GD estimator is not. For both linear and logistic regression, we show the performance of LGD run with burn-in $B=500$, $b=5000$ averaging steps and learning rate $\eta = 9e-4$, and GD with a total of $5500$ iterations and $\eta = 9e-4$ in Figures \ref{fig:lgd_plots}, \ref{fig:lgd_class_plots}. We note that the oracle LGD performs on par with the oracle GD for linear regression as expected, and only slightly better than oracle GD for logistic regression because of the relatively large regularization. GD without regularization matches the performance of the oracles around 15 training samples for linear regression and around 100 training samples for logistic regression due to the difficulty of the tasks. This also indicates that few-shot learning is beneficial in a larger regime for classification tasks. We additionally plot LGD using meta-learned $\theta$ and learning rate $\eta$ with same values of $B,b$ for both linear and logistic regression. Meta learning is done using the Adam optimizer \citep{Adam} with a learning rate of $0.4$ and $50$ steps with the log-values of $\theta$ and $\eta$ as parameters for stability, and clamping $\eta$ in the range $[1e-7,1.2e-3]$. We observe that the meta-learned LGD matches the oracle baselines even using 1 training sample per task for both linear and logistic regression. Note that practically, when the underlying prior is unknown, the oracle algorithms are unachievable whereas the meta-learned algorithm is still realizable.

\subsection{Non-Isotropic Gaussian}
We extend the above prior to a non-isotropic Gaussian. We first sample $d=10$ variance values $v_1,\ldots, v_d$ uniformly between $0.05$ and $0.5$ and create the diagonal covariance matrix $\Sigma = \operatorname{diag}(v_1,\ldots, v_d)$. We now use this covariance matrix to sample the ground truth for each task, $w^{*t}\sim N(0,\Sigma)$. Similar to above, the Bayes' optimal estimator for linear regression for this prior is the mean of the posterior distributed as $Pr(w|X,y)\propto \exp{\left(-\frac{\|y-Xw\|^2}{2} - \left(\sum \frac{w_i^2}{2v_i}\right) \right)}$. Since this is again a Gaussian posterior, the mean and the mode are equal and equal to the minimizer of $\|Xw-y\|^2 + \sum \frac{w_i^2}{v_i}$. Similar to above, the posterior of $w$ is not Gaussian for logistic regression, and hence the LGD estimator is Bayes' optimal while the GD estimator is not. For both linear and logistic regression, we show the performance of LGD run with burn-in $B=500$, $b=5000$ averaging steps and learning rate $\eta = 9e-4$ using the gradient of regularizer $r(w;\theta) = \sum w_i\theta_i$ for $\theta^* = (v_1,\ldots,v_d)$, as well as GD run with the same regularizer for $5500$ iterations and $\eta = 9e-4$ in Figures \ref{fig:lgd_plots}, \ref{fig:lgd_class_plots}. We observe that the oracle LGD algorithm performs slightly worse than oracle GD due to noise in the algorithm and a sub-optimal mixing time for linear regression. On the contrary, oracle LGD performs at par or slightly better than oracle GD for logistic regression since the oracle GD is no longer the Bayes' optimal. Similar to the isotropic Gaussian prior, GD without regularization matches the performance of the oracles around 15 training samples for linear regression and around 100 training samples for logistic regression due to the difficulty of the tasks. Lastly, we also plot the performance of the meta-learned LGD algorithm for both linear and logistic regression, where we again use the Adam optimizer \citep{Adam} with learning rate $0.4$ using $50$ steps, and optimize the log of $\theta$ and $\eta$ while clamping $\eta$ in the range $[1e-7,1.2e-3]$. We observe that the meta-learned LGD matches the performance of the oracle LGD even with 1 training example per task for both linear and logistic regression. We note again that the oracle algorithms are unachievable in practice, while the meta-learned algorithm is.

\subsection{Gaussian with Softplus Scaling}
Finally, we conduct experiments on a prior with a clear separation between mean and mode, to better illustrate the advantages of LGD over plain GD. Towards this, we study the prior
\begin{align*}
    \pi_{\lambda, \beta, \gamma}(u)&\propto \exp{\left(-\frac{\gamma u^2}{2} -\lambda\log(1+\exp(\beta u))\right)} = \exp{\left(-\frac{\gamma u^2}{2}\right)}\left(\frac{1}{1+\exp(\beta u)}\right)^\lambda.
\end{align*}
While this prior is still log-strongly-concave as required by the conditions of Theorem \ref{thm:lgd_conv}, the mean of the prior is well-separated from the mode. This creates a posterior distribution where the mean and mode are separated as well, and thus LGD performs better than GD for both linear and logistic regression. Formally, first note that we sample each element of each ground truth vector from this prior, $w^{*t(i)}\sim \pi_{1,10,0.1}$. We then construct the gradient of regularizer $r(w;\theta) = \theta[0]w + \beta\theta[1]\left(\frac{1}{1 + \exp(-\beta w)}\right)$, where we treat $\beta$ as a constant for simplicity. We plot the performance of LGD run with $\theta^*=(1, 0.1)$ with burn-in $B=5000$, $b=50000$ averaging steps and learning rate $\eta = 1e-4$, and the oracle GD run with a total of $55000$ steps and $\eta = 1e-4$ in Figures \ref{fig:lgd_plots}, \ref{fig:lgd_class_plots}. For both linear and logistic regression, we see that the oracle LGD beats the oracle GD algorithm and massively out-performs the GD algorithm without regularization. We additionally plot the meta-learned LGD algorithm, where meta learning is done using the Adam optimizer \citep{Adam} with learning rate $0.4$ using $50$ steps, and optimize the log of $\theta$ and $\eta$ while clamping $\eta$ in the range $[1e-7,1e-4]$. We observe that the meta-learned algorithm manages to outperform both of these algorithms by a wide margin. We hypothesize that the LGD oracle performs worse because the poor condition number of the Hessian of the log-prior requires a higher mixing time, and the meta-learned algorithm finds a different prior with a lower mixing time. To confirm this, we plot another oracle LGD using $10\times$ the values of $B$ and $b$ and $\eta = 1e-4$. We see that the meta-learned algorithm performs almost on par with the longer oracle. Thus, the meta-learned algorithm manages to perform on par with the optimal algorithm using $1/10$th the number of iterations! As mentioned in previous Sections, none of the baselines other than gradient descent without regularization are achievable in practice (we can expect hand-tuned gradient descent to perform somewhere between GD with no regularization and the oracle GD), which illustrates the advantages of using our approach. 

\section{Discussion and Future Work}
In this work we study the generalized setting of learning hyperparameters for gradient descent with a convex objective. We propose the Langevin Gradient Descent algorithm which performs Langevin-style updates on gradient descent and uses these updates to directly estimate the outputs for test/validation samples. We show that the Langevin-style sampling followed by a consolidation involving averaging predictions of the final iterates results in Bayes' optimality of the algorithm's estimates under mild assumptions. We additionally show generalization guarantees for finding the hyperparameters corresponding to regularization and learning rate from multiple tasks in a data-driven setting using pseudo-dimension based arguments. While our setting extends the setting of prior work on regression using the elastic net significantly, our bounds match the pseudo-dimension bounds for elastic net up to logarithmic terms. On the other hand, we observe a worse dependence of sample complexity on generalization error while accounting for imperfect optimization and Bayes' optimality for generic priors. We similarly extend the setting of prior work on regularized logistic regression to much more general models and regularizers. Reducing to the setting of prior work, we observe a much tighter dependence of pseudo-dimension on the number of samples per task owing to our consideration of hardware-aware procedures. We observe a worse dependence on the approximation error and number of parameters. We attribute these degradations to our choice of a generic sampling-based algorithm that easily extends beyond the linear case, whereas prior work restricts their attention to optimization-based algorithms that don't extend beyond the linear case. We then provide experimental results on three different synthetic datasets for linear regression and three different synthetic datasets for logistic regression. Our results demonstrate the massive advantage that the proposed Langevin Gradient Descent algorithm, and meta learning provide for few-shot learning of regression with convex objectives. We observe an even greater advantage with few-shot learning of classification since our algorithms perform better than algorithms without regularization for a wider range of sample sizes.
Exploring data-dependent bounds using Rademacher complexities is an interesting direction for future research. Another interesting direction would be to explore extensions of the algorithm for multi-modal priors, as opposed to strongly convex priors, and correspondingly extend the framework to the analysis of gradient descent on neural networks. On the empirical side, experiments on non-linear models for regression and classification would be directions for future research.

\bibliographystyle{plainnat}
\bibliography{ref}

@misc{engstrom2025optimizingmltrainingmetagradient,
      title={Optimizing ML Training with Metagradient Descent}, 
      author={Logan Engstrom and Andrew Ilyas and Benjamin Chen and Axel Feldmann and William Moses and Aleksander Madry},
      year={2025},
      eprint={2503.13751},
      archivePrefix={arXiv},
      primaryClass={stat.ML},
      url={https://arxiv.org/abs/2503.13751}, 
}

@article{li2018hyperbandnovelbanditbasedapproach,
  author  = {Lisha Li and Kevin Jamieson and Giulia DeSalvo and Afshin Rostamizadeh and Ameet Talwalkar},
  title   = {Hyperband: A Novel Bandit-Based Approach to Hyperparameter Optimization},
  journal = {Journal of Machine Learning Research},
  year    = {2018},
  volume  = {18},
  number  = {185},
  pages   = {1--52},
}

@InProceedings{finn2017modelagnosticmetalearningfastadaptation,
  title = 	 {Model-Agnostic Meta-Learning for Fast Adaptation of Deep Networks},
  author =       {Chelsea Finn and Pieter Abbeel and Sergey Levine},
  booktitle = 	 {Proceedings of the 34th International Conference on Machine Learning},
  pages = 	 {1126--1135},
  year = 	 {2017},
  editor = 	 {Precup, Doina and Teh, Yee Whye},
  volume = 	 {70},
  series = 	 {Proceedings of Machine Learning Research},
  month = 	 {06--11 Aug},
  publisher =    {PMLR},
}

@article{polyak_92,
author = {Polyak, B. T. and Juditsky, A. B.},
title = {Acceleration of Stochastic Approximation by Averaging},
journal = {SIAM Journal on Control and Optimization},
volume = {30},
number = {4},
pages = {838-855},
year = {1992},
doi = {10.1137/0330046}
}

@article{gj_95,
author = {Goldberg, Paul W. and Jerrum, Mark R.},
title = {Bounding the Vapnik-Chervonenkis Dimension of Concept Classes Parameterized by Real Numbers},
year = {1995},
issue_date = {Feb./March 1995},
publisher = {Kluwer Academic Publishers},
address = {USA},
volume = {18},
number = {2–3},
issn = {0885-6125},
doi = {10.1007/BF00993408},
journal = {Machine Learning},
month = feb,
pages = {131–148},
numpages = {18}
}

@inproceedings{balcan_23,
 author = {Balcan, Maria-Florina F and Nguyen, Anh and Sharma, Dravyansh},
 booktitle = {Advances in Neural Information Processing Systems},
 editor = {A. Oh and T. Naumann and A. Globerson and K. Saenko and M. Hardt and S. Levine},
 pages = {80066--80078},
 publisher = {Curran Associates, Inc.},
 title = {New Bounds for Hyperparameter Tuning of Regression Problems Across Instances},
 volume = {36},
 year = {2023}
}

@InProceedings{bartlett_linalg_22,
  title = 	 {Generalization Bounds for Data-Driven Numerical Linear Algebra},
  author =       {Bartlett, Peter and Indyk, Piotr and Wagner, Tal},
  booktitle = 	 {Proceedings of Thirty Fifth Conference on Learning Theory},
  pages = 	 {2013--2040},
  year = 	 {2022},
  editor = 	 {Loh, Po-Ling and Raginsky, Maxim},
  volume = 	 {178},
  series = 	 {Proceedings of Machine Learning Research},
  month = 	 {02--05 Jul},
  publisher =    {PMLR},
}

@misc{goyal_25,
      title={Distribution-dependent Generalization Bounds for Tuning Linear Regression Across Tasks}, 
      author={Maria-Florina Balcan and Saumya Goyal and Dravyansh Sharma},
      year={2025},
      eprint={2507.05084},
      archivePrefix={arXiv},
      primaryClass={cs.LG},
}

@article{durmus_ula_19,
author = {Alain Durmus and {\'E}ric Moulines},
title = {{High-dimensional Bayesian inference via the unadjusted Langevin algorithm}},
volume = {25},
journal = {Bernoulli},
number = {4A},
publisher = {Bernoulli Society for Mathematical Statistics and Probability},
pages = {2854 -- 2882},
year = {2019},
doi = {10.3150/18-BEJ1073},
}

@inproceedings{enet_2022,
 author = {Balcan, Maria-Florina and Khodak, Misha and Sharma, Dravyansh and Talwalkar, Ameet},
 booktitle = {Advances in Neural Information Processing Systems},
 editor = {S. Koyejo and S. Mohamed and A. Agarwal and D. Belgrave and K. Cho and A. Oh},
 pages = {27769--27782},
 publisher = {Curran Associates, Inc.},
 title = {Provably tuning the {ElasticNet} across instances},
 volume = {35},
 year = {2022}
}

@book{nn_theory_bartlett_99,
place={Cambridge}, 
title={Neural Network Learning: Theoretical Foundations}, 
publisher={Cambridge University Press}, 
author={Anthony, Martin and Bartlett, Peter L.}, 
year={1999}}

@article{JMLR:v13:bergstra12a,
  author  = {James Bergstra and Yoshua Bengio},
  title   = {Random Search for Hyper-Parameter Optimization},
  journal = {Journal of Machine Learning Research},
  year    = {2012},
  volume  = {13},
  number  = {10},
  pages   = {281--305},
}

@inproceedings{snoek2012practicalbayesianoptimizationmachine,
 author = {Snoek, Jasper and Larochelle, Hugo and Adams, Ryan P},
 booktitle = {Advances in Neural Information Processing Systems},
 editor = {F. Pereira and C.J. Burges and L. Bottou and K.Q. Weinberger},
 pages = {},
 publisher = {Curran Associates, Inc.},
 title = {Practical Bayesian Optimization of Machine Learning Algorithms},
 volume = {25},
 year = {2012}
}

@inproceedings{NIPS2011_86e8f7ab,
 author = {Bergstra, James and Bardenet, R\'{e}mi and Bengio, Yoshua and K\'{e}gl, Bal\'{a}zs},
 booktitle = {Advances in Neural Information Processing Systems},
 editor = {J. Shawe-Taylor and R. Zemel and P. Bartlett and F. Pereira and K.Q. Weinberger},
 pages = {},
 publisher = {Curran Associates, Inc.},
 title = {Algorithms for Hyper-Parameter Optimization},
 volume = {24},
 year = {2011}
}

@misc{jaderberg2017populationbasedtrainingneural,
      title={Population Based Training of Neural Networks}, 
      author={Max Jaderberg and Valentin Dalibard and Simon Osindero and Wojciech M. Czarnecki and Jeff Donahue and Ali Razavi and Oriol Vinyals and Tim Green and Iain Dunning and Karen Simonyan and Chrisantha Fernando and Koray Kavukcuoglu},
      year={2017},
      eprint={1711.09846},
      archivePrefix={arXiv},
      primaryClass={cs.LG},
      url={https://arxiv.org/abs/1711.09846}, 
}

@article{nn_pd_bartlett19,
  author  = {Peter L. Bartlett and Nick Harvey and Christopher Liaw and Abbas Mehrabian},
  title   = {Nearly-tight VC-dimension and Pseudodimension Bounds for Piecewise Linear Neural Networks},
  journal = {Journal of Machine Learning Research},
  year    = {2019},
  volume  = {20},
  number  = {63},
  pages   = {1--17},
}

@inproceedings{balcan2020data,
  title={{Data-Driven Algorithm Design} (Book Chapter)},
  author={Maria-Florina Balcan},
  booktitle={Beyond Worst-Case Analysis of Algorithms, Tim Roughgarden (Ed)},
  year={2020},
  publisher={{Cambridge University Press}}
}

@InProceedings{hyperparam_grad_dougal15,
  title = 	 {Gradient-based Hyperparameter Optimization through Reversible Learning},
  author = 	 {Maclaurin, Dougal and Duvenaud, David and Adams, Ryan},
  booktitle = 	 {Proceedings of the 32nd International Conference on Machine Learning},
  pages = 	 {2113--2122},
  year = 	 {2015},
  editor = 	 {Bach, Francis and Blei, David},
  volume = 	 {37},
  series = 	 {Proceedings of Machine Learning Research},
  address = 	 {Lille, France},
  month = 	 {07--09 Jul},
  publisher =    {PMLR},
}

@inproceedings{
hypergrad_baydin18,
title={Online Learning Rate Adaptation with Hypergradient Descent},
author={Atilim Gunes Baydin and Robert Cornish and David Martinez Rubio and Mark Schmidt and Frank Wood},
booktitle={International Conference on Learning Representations},
year={2018},
}

@inproceedings{gd_by_gd_16,
 author = {Andrychowicz, Marcin and Denil, Misha and G\'{o}mez, Sergio and Hoffman, Matthew W and Pfau, David and Schaul, Tom and Shillingford, Brendan and de Freitas, Nando},
 booktitle = {Advances in Neural Information Processing Systems},
 editor = {D. Lee and M. Sugiyama and U. Luxburg and I. Guyon and R. Garnett},
 pages = {},
 publisher = {Curran Associates, Inc.},
 title = {Learning to learn by gradient descent by gradient descent},
 volume = {29},
 year = {2016}
}

@inproceedings{
datarater_calian25,
title={DataRater: Meta-Learned Dataset Curation},
author={Dan A. Calian and Gregory Farquhar and Iurii Kemaev and Luisa Zintgraf and Matteo Hessel and Jeremy Shar and Junhyuk Oh and Andr{\'a}s Gy{\"o}rgy and Tom Schaul and Jeff Dean and Hado van Hasselt and David Silver},
booktitle={The Thirty-ninth Annual Conference on Neural Information Processing Systems},
year={2025},
}

@misc{nichol2018firstordermetalearningalgorithms,
      title={On First-Order Meta-Learning Algorithms}, 
      author={Alex Nichol and Joshua Achiam and John Schulman},
      year={2018},
      eprint={1803.02999},
      archivePrefix={arXiv},
      primaryClass={cs.LG},
      url={https://arxiv.org/abs/1803.02999}, 
}

@misc{franceschi2018bilevelprogramminghyperparameteroptimization,
      title={Bilevel Programming for Hyperparameter Optimization and Meta-Learning}, 
      author={Luca Franceschi and Paolo Frasconi and Saverio Salzo and Riccardo Grazzi and Massimilano Pontil},
      year={2018},
      eprint={1806.04910},
      archivePrefix={arXiv},
      primaryClass={stat.ML},
      url={https://arxiv.org/abs/1806.04910}, 
}

@misc{rajeswaran2019metalearningimplicitgradients,
      title={Meta-Learning with Implicit Gradients}, 
      author={Aravind Rajeswaran and Chelsea Finn and Sham Kakade and Sergey Levine},
      year={2019},
      eprint={1909.04630},
      archivePrefix={arXiv},
      primaryClass={cs.LG},
      url={https://arxiv.org/abs/1909.04630}, 
}

@article{liu2021investigatingbileveloptimizationlearning,
  author={Liu, Risheng and Gao, Jiaxin and Zhang, Jin and Meng, Deyu and Lin, Zhouchen},
  journal={IEEE Transactions on Pattern Analysis and Machine Intelligence}, 
  title={Investigating Bi-Level Optimization for Learning and Vision From a Unified Perspective: A Survey and Beyond}, 
  year={2022},
  volume={44},
  number={12},
  pages={10045-10067},
  doi={10.1109/TPAMI.2021.3132674}
}

@misc{grattafiori2024llama3herdmodels,
      title={The Llama 3 Herd of Models}, 
      author={{Meta AI}},
      year={2024},
      eprint={2407.21783},
      archivePrefix={arXiv},
      primaryClass={cs.AI},
}

@InProceedings{pmlr-v97-balcan19a,
  title = 	 {Provable Guarantees for Gradient-Based Meta-Learning},
  author =       {Balcan, Maria-Florina and Khodak, Mikhail and Talwalkar, Ameet},
  booktitle = 	 {Proceedings of the 36th International Conference on Machine Learning},
  pages = 	 {424--433},
  year = 	 {2019},
  editor = 	 {Chaudhuri, Kamalika and Salakhutdinov, Ruslan},
  volume = 	 {97},
  series = 	 {Proceedings of Machine Learning Research},
  month = 	 {09--15 Jun},
  publisher =    {PMLR},
}

@inproceedings{NEURIPS2019_f4aa0dd9,
 author = {Khodak, Mikhail and Balcan, Maria-Florina F and Talwalkar, Ameet S},
 booktitle = {Advances in Neural Information Processing Systems},
 editor = {H. Wallach and H. Larochelle and A. Beygelzimer and F. d\textquotesingle Alch\'{e}-Buc and E. Fox and R. Garnett},
 pages = {},
 publisher = {Curran Associates, Inc.},
 title = {Adaptive Gradient-Based Meta-Learning Methods},
 volume = {32},
 year = {2019}
}

@book{all_of_stats,
  address = {New York},
  author = {Wasserman, Larry},
  description = {All of Statistics: A Concise Course in Statistical Inference},
  publisher = {Springer},
  title = {All of Statistics},
  year = 2010
}

@article{l20_chen22,
  author  = {Tianlong Chen and Xiaohan Chen and Wuyang Chen and Howard Heaton and Jialin Liu and Zhangyang Wang and Wotao Yin},
  title   = {Learning to Optimize: A Primer and A Benchmark},
  journal = {Journal of Machine Learning Research},
  year    = {2022},
  volume  = {23},
  number  = {189},
  pages   = {1--59},
}

@misc{hpa_balcan_21,
      title={How much data is sufficient to learn high-performing algorithms? Generalization guarantees for data-driven algorithm design}, 
      author={Maria-Florina Balcan and Dan DeBlasio and Travis Dick and Carl Kingsford and Tuomas Sandholm and Ellen Vitercik},
      year={2021},
      eprint={1908.02894},
      archivePrefix={arXiv},
      primaryClass={cs.LG},
}

@book{ML_theory_Shwartz_14,
place={Cambridge}, 
title={Understanding Machine Learning: From Theory to Algorithms}, 
publisher={Cambridge University Press}, 
author={Shalev-Shwartz, Shai and Ben-David, Shai},
year={2014}}

@article{Adam,
  title={Adam: A Method for Stochastic Optimization},
  author={Diederik P. Kingma and Jimmy Ba},
  journal={CoRR},
  year={2014},
  volume={abs/1412.6980},
}

@article{grenander1983tutorial,
  title={Tutorial in pattern theory},
  author={Grenander, Ulf},
  journal={Report, Division of Applied Mathematics},
  year={1983},
  publisher={Brown University}
}

@misc{etde_5964505,
title = {Correlation functions and computer simulations},
author = {Parisi, G},
abstractNote = {If the equilibrium properties of a statistical system are obtained by solving numerically the associated Langevin equation describing the approach to equilibrium, the connected correlation functions can be computed directly with small effort and high precision.},
journal = {},
volume = {180:3},
place = {Netherlands},
year = {1981},
month = {May}
}

@inproceedings{Welling2011BayesianLV,
  title={Bayesian Learning via Stochastic Gradient Langevin Dynamics},
  author={Welling, Max and Teh, Yee Whye},
  booktitle={Proceedings of the 28th International Conference on Machine Learning (ICML-11)},
  pages={681--688},
  year={2011},
  publisher={Omnipress}
}

@InProceedings{pmlr-v65-dalalyan17a,
  title = 	 {Further and stronger analogy between sampling and optimization: Langevin Monte Carlo and gradient descent},
  author = 	 {Dalalyan, Arnak},
  booktitle = 	 {Proceedings of the 2017 Conference on Learning Theory},
  pages = 	 {678--689},
  year = 	 {2017},
  editor = 	 {Kale, Satyen and Shamir, Ohad},
  volume = 	 {65},
  series = 	 {Proceedings of Machine Learning Research},
  month = 	 {07--10 Jul},
  publisher =    {PMLR},
}

@article{durmus2016nonasymptoticconvergenceanalysisunadjusted,
author = {Alain Durmus and {\'E}ric Moulines},
title = {{Nonasymptotic convergence analysis for the unadjusted Langevin algorithm}},
volume = {27},
journal = {The Annals of Applied Probability},
number = {3},
publisher = {Institute of Mathematical Statistics},
pages = {1551 -- 1587},
year = {2017},
doi = {10.1214/16-AAP1238},
}

@InProceedings{raginsky2017nonconvexlearningstochasticgradient,
  title = 	 {Non-convex learning via Stochastic Gradient Langevin Dynamics: a nonasymptotic analysis},
  author = 	 {Raginsky, Maxim and Rakhlin, Alexander and Telgarsky, Matus},
  booktitle = 	 {Proceedings of the 2017 Conference on Learning Theory},
  pages = 	 {1674--1703},
  year = 	 {2017},
  editor = 	 {Kale, Satyen and Shamir, Ohad},
  volume = 	 {65},
  series = 	 {Proceedings of Machine Learning Research},
  month = 	 {07--10 Jul},
  publisher =    {PMLR},
}

@inproceedings{xu2020globalconvergencelangevindynamics,
 author = {Xu, Pan and Chen, Jinghui and Zou, Difan and Gu, Quanquan},
 booktitle = {Advances in Neural Information Processing Systems},
 editor = {S. Bengio and H. Wallach and H. Larochelle and K. Grauman and N. Cesa-Bianchi and R. Garnett},
 pages = {},
 publisher = {Curran Associates, Inc.},
 title = {Global Convergence of Langevin Dynamics Based Algorithms for Nonconvex Optimization},
 volume = {31},
 year = {2018}
}

@inproceedings{Dubey2016VarianceRI,
 author = {Dubey, Kumar Avinava and J. Reddi, Sashank and Williamson, Sinead A and Poczos, Barnabas and Smola, Alexander J and Xing, Eric P},
 booktitle = {Advances in Neural Information Processing Systems},
 editor = {D. Lee and M. Sugiyama and U. Luxburg and I. Guyon and R. Garnett},
 pages = {},
 publisher = {Curran Associates, Inc.},
 title = {Variance Reduction in Stochastic Gradient Langevin Dynamics},
 volume = {29},
 year = {2016}
}

@inproceedings{ma2015completerecipestochasticgradient,
 author = {Ma, Yi-An and Chen, Tianqi and Fox, Emily},
 booktitle = {Advances in Neural Information Processing Systems},
 editor = {C. Cortes and N. Lawrence and D. Lee and M. Sugiyama and R. Garnett},
 pages = {},
 publisher = {Curran Associates, Inc.},
 title = {A Complete Recipe for Stochastic Gradient MCMC},
 volume = {28},
 year = {2015}
}

@article{https://doi.org/10.1111/j.1467-9868.2010.00765.x,
author = {Girolami, Mark and Calderhead, Ben},
title = {Riemann manifold Langevin and Hamiltonian Monte Carlo methods},
journal = {Journal of the Royal Statistical Society: Series B (Statistical Methodology)},
volume = {73},
number = {2},
pages = {123-214},
year = {2011}
}

@article{nemeth2019stochasticgradientmarkovchain,
author = {Christopher Nemeth and Paul Fearnhead},
title = {Stochastic Gradient Markov Chain Monte Carlo},
journal = {Journal of the American Statistical Association},
volume = {116},
number = {533},
pages = {433--450},
year = {2021},
publisher = {Taylor \& Francis},
doi = {10.1080/01621459.2020.1847120},
}

@inproceedings{
chung2024diffusionposteriorsamplinggeneral,
title={Diffusion Posterior Sampling for General Noisy Inverse Problems},
author={Hyungjin Chung and Jeongsol Kim and Michael Thompson Mccann and Marc Louis Klasky and Jong Chul Ye},
booktitle={The Eleventh International Conference on Learning Representations },
year={2023},
}

@inproceedings{song2020generativemodelingestimatinggradients,
 author = {Song, Yang and Ermon, Stefano},
 booktitle = {Advances in Neural Information Processing Systems},
 editor = {H. Wallach and H. Larochelle and A. Beygelzimer and F. d\textquotesingle Alch\'{e}-Buc and E. Fox and R. Garnett},
 pages = {},
 publisher = {Curran Associates, Inc.},
 title = {Generative Modeling by Estimating Gradients of the Data Distribution},
 volume = {32},
 year = {2019}
}

@inproceedings{
song2021scorebased,
title={Score-Based Generative Modeling through Stochastic Differential Equations},
author={Yang Song and Jascha Sohl-Dickstein and Diederik P Kingma and Abhishek Kumar and Stefano Ermon and Ben Poole},
booktitle={International Conference on Learning Representations},
year={2021},
}

@book{nesterov2018lectures,
  title     = {Lectures on Convex Optimization},
  author    = {Nesterov, Yurii},
  year      = {2018},
  publisher = {Springer},
  address   = {Cham, Switzerland},
  series    = {Springer Optimization and Its Applications}
}

@InProceedings{offset_rad,
  title = 	 {Learning with Square Loss: Localization through Offset Rademacher Complexity},
  author = 	 {Liang, Tengyuan and Rakhlin, Alexander and Sridharan, Karthik},
  booktitle = 	 {Proceedings of The 28th Conference on Learning Theory},
  pages = 	 {1260--1285},
  year = 	 {2015},
  editor = 	 {Grünwald, Peter and Hazan, Elad and Kale, Satyen},
  volume = 	 {40},
  series = 	 {Proceedings of Machine Learning Research},
  address = 	 {Paris, France},
  month = 	 {03--06 Jul},
  publisher =    {PMLR},
}

@Book{fp_arithmetic_muller_18,
  title        = {Handbook of Floating-Point Arithmetic, 2nd edition},
  author    = {Muller, Jean-Michel and Brunie, Nicolas and de Dinechin, Florent and Jeannerod, Claude-Pierre and Joldes, Mioara and
          Lef{\`e}vre, Vincent and Melquiond, Guillaume and Revol,
          Nathalie and  Torres, Serge},
  publisher    = {{B}irkh\"auser {B}oston },
  pages    = {632},
  note        = {{ACM} {G}.1.0; {G}.1.2; {G}.4; {B}.2.0; {B}.2.4; {F}.2.1.,
                  ISBN 978-3-319-76525-9},
  year        = {2018},
}

@article{
nair2026softmax,
title={Softmax is \$1/2\$-Lipschitz: A tight bound across all \${\textbackslash}ell\_p\$ norms},
author={Pravin Nair},
journal={Transactions on Machine Learning Research},
issn={2835-8856},
year={2026},
}

@InProceedings{kontorovich13,
  title = 	 {Concentration in unbounded metric spaces and algorithmic stability},
  author = 	 {Kontorovich, Aryeh},
  booktitle = 	 {Proceedings of the 31st International Conference on Machine Learning},
  pages = 	 {28--36},
  year = 	 {2014},
  editor = 	 {Xing, Eric P. and Jebara, Tony},
  volume = 	 {32},
  series = 	 {Proceedings of Machine Learning Research},
  address = 	 {Bejing, China},
  month = 	 {22--24 Jun},
  publisher =    {PMLR},
}

@article{galvez24,
  author  = {Borja Rodr{{\'i}}guez-G{{\'a}}lvez and Ragnar Thobaben and Mikael Skoglund},
  title   = {More {PAC-B}ayes bounds: From bounded losses, to losses with general tail behaviors, to anytime validity},
  journal = {Journal of Machine Learning Research},
  year    = {2024},
  volume  = {25},
  number  = {110},
  pages   = {1--43},
}

@inproceedings{rosset_04,
 author = {Rosset, Saharon},
 booktitle = {Advances in Neural Information Processing Systems},
 editor = {L. Saul and Y. Weiss and L. Bottou},
 pages = {},
 publisher = {MIT Press},
 title = {Following Curved Regularized Optimization Solution Paths},
 volume = {17},
 year = {2004}
}

@InProceedings{poly_approx_vitercik2020,
  title = 	 {Refined bounds for algorithm configuration: The knife-edge of dual class approximability},
  author =       {Balcan, Maria-Florina and Sandholm, Tuomas and Vitercik, Ellen},
  booktitle = 	 {Proceedings of the 37th International Conference on Machine Learning},
  pages = 	 {580--590},
  year = 	 {2020},
  editor = 	 {III, Hal Daumé and Singh, Aarti},
  volume = 	 {119},
  series = 	 {Proceedings of Machine Learning Research},
  month = 	 {13--18 Jul},
  publisher =    {PMLR},
}

@incollection{gabrielov2004complexity,
  author    = {Gabrielov, Andrei and Vorobjov, Nicolai},
  title     = {Complexity of computations with Pfaffian and Noetherian functions},
  booktitle = {Normal Forms, Bifurcations and Finiteness Problems in Differential Equations},
  editor    = {Ilyashenko, Y. and Rousseau, C.},
  series    = {NATO Science Series II: Mathematics, Physics and Chemistry},
  volume    = {137},
  pages     = {211--250},
  publisher = {Kluwer Academic Publishers},
  address   = {Dordrecht},
  year      = {2004}
}

\appendix

\section{Additional Related Work}\label{sec:add_rel_work}
Proving generalization guarantees on learning algorithms using pseudo-dimension is a popular approach in machine learning literature \citep{nn_theory_bartlett_99}. There are several approaches to bound the pseudo-dimension of a function class, like covering-number based arguments. In this work, we use the algorithmic framework of \citet{gj_95}, which upper bounds the pseudo-dimension of a function class $\mcF$ by constructing an algorithm $\Gamma_{x,r}$ to calculate whether $f(x)\ge r$ for any input $f\in\mcF$. \citet{nn_pd_bartlett19}  use this approach to show an almost-tight upper bound on the pseudo-dimension of a neural network class using ReLU activation with $W$ number of parameters and $D$ layers to be $O(WD\log(W))$. Building on these results, we assume that the predicate complexity and degree of computing the given functions is known and use these values to calculate the pseudo-dimension of the validation loss function class. 

The foundational idea of using gradient based diffusion processes to sample from a target density traces its roots from statistical physics literature, but it was introduced in modern statistics by~\cite{grenander1983tutorial}, and formalized as the Unadjusted Langevin Algorithm (ULA) by~\cite{etde_5964505}. The adoption in machine learning was encouraged by \citet{Welling2011BayesianLV}, who proposed stochastic gradient Langevin dynamics (SGLD), replacing exact gradients with minibatch stochastic gradients and eliminating the Metropolis correction for scalability to large datasets. Non-asymptotic convergence guarantees for the ULA in Wasserstein and TV distance for log-concave distributions were established  by~\cite{pmlr-v65-dalalyan17a} and~\cite{durmus2016nonasymptoticconvergenceanalysisunadjusted}. \citet{raginsky2017nonconvexlearningstochasticgradient, xu2020globalconvergencelangevindynamics} analyse the convergence of Langevin-based gradient algorithms to minimizers of the loss function for non-convex settings by assuming dissipativity of the energy function. Further refinements such as preconditioning with curvature information~\cite{https://doi.org/10.1111/j.1467-9868.2010.00765.x}, stochastic gradient MCMC frameworks unifying Langevin and Hamiltonian variants~\cite{ma2015completerecipestochasticgradient}, and variance-reduced estimators such as SVRG-Langevin~\cite{Dubey2016VarianceRI} were developed, which improve convergence rates while retaining the minibatch scalability, thus making Langevin methods attractive for modern Bayesian deep learning and large-scale latent variable models. As a result, Langevin dynamics has now become a central tool for scalable posterior sampling in Bayesian inference and in the diffusion based generative modeling literature~\citep{nemeth2019stochasticgradientmarkovchain, chung2024diffusionposteriorsamplinggeneral, song2020generativemodelingestimatinggradients, song2021scorebased}.

\section{Convergence Guarantees on Langevin Algorithms}\label{sec:asymp_lgd}


Recall the update step for Langevin diffusion as presented in Section \ref{sec:bckg}. For a distribution $\pi \propto \exp(-U)$ defined over $\R^d$, Langevin diffusion follows a continuous process governed by the following Stochastic Differential Equation (SDE):
\begin{align*}
    \dd z_t = -\nabla U(z_t)\dd t + \sqrt{2}\dd \xi_t,
\end{align*}
where $\xi_t$ corresponds to Brownian motion in $\R^d$. In practice, we simulate this continuous process using a discretization called the Unadjusted Langevin Algorithm (ULA). Starting with a sample $Z_{(0)}$, ULA constitutes the following iterative updates:
\begin{align*}
    z_{(k+1)} = z_{(k)} - \eta_{k + 1}\nabla U+\sqrt{2\eta_{k + 1}} \xi_k;\quad \xi_k\sim \mcN(0,\mathbb{I}).
\end{align*}

For simplicity, we also state ULA with constant step sizes, where $\eta_k = \eta \quad \forall k$. We call this the homogeneous ULA that performs the following iterative updates to produce samples from $\pi \propto \exp(-U)$:
\begin{align*}
    z_{(k+1)} = z_{(k)} - \eta\nabla U+\sqrt{2\eta} \xi_k;\quad \xi_k\sim \mcN(0,\mathbb{I}).
\end{align*}

We will use results from \citep{durmus_ula_19} which provide a detailed analysis of homogeneous and inhomogeneous ULA. We recall the following two assumptions:
\begin{app_assumption}\label{appass:grad_lip}
    The function $U$ is continuously differentiable on $\R^d$ and gradient Lipschitz: there exists $\beta\ge 0$ such that for all $z_1,z_2\in\R^d$, $\|\nabla U(z_1)-\nabla U(z_2)\| \le \beta\|z_1-z_2\|$.
\end{app_assumption}
\begin{app_assumption}\label{appass:strong_conv}
    $U$ is strongly convex, i.e. there exists $\alpha > 0$ such that for all $z_1,z_2\in \R^d$,
    \begin{align*}
        U(z_2) \ge U(z_1) + \langle\nabla U(z_1), z_2-z_1\rangle + (\alpha/2)\|z_1-z_2\|^2.
    \end{align*}
\end{app_assumption}
Note that for both of these assumptions to hold, $\alpha\le \beta$. In the following, assume $\kappa = \frac{2\alpha\beta}{\alpha + \beta}$, and the sequence of step-sizes $\eta_k$ is non-increasing with $\eta_1 \le 1/(\alpha+\beta)$. Lemma \ref{lem:wass_ub} and Proposition \ref{lem:ula_conv} show that ULA samples from a stationary distribution that is close to the target distribution as measured by the Wasserstein-2 metric, which is defined below:
\begin{definition}[Wasserstein-2 metric]
    For two distributions $\pi,\pi'$, and $\Pi(\pi,\pi')$ being the set of all couplings between the two distributions, we define the Wasserstein-2 ($W_2$) distance as:
    \begin{align*}
        W_2^2(\pi,\pi') = \inf_{\gamma\in \Pi(\pi,\pi')} \E_{z,z' \sim \gamma}\left[\|z-z'\|^2\right].
    \end{align*}
\end{definition}

We get the following result on the W2 distance between the distribution of $z_{(k)}$, denoted $\pi^k$, and target distribution of samples, $\pi$.

\begin{lemma}[Thm 5, \citep{durmus_ula_19}]\label{lem:wass_ub}
    Under the assumptions \ref{appass:grad_lip} and \ref{appass:strong_conv}, assuming we start the ULA from $\pi^0 = \delta_z$ for some $z\in \R^d$, we get that for any $k\ge 1$,
    \begin{align*}
        W_2^2(\pi^k, \pi) \le u_k^{(1)}(\eta)\{\|z-z^*\|^2 + d/\alpha\} + u_k^{(2)}(\eta),
    \end{align*}
    where,
    \begin{align*}
        u_k^{(1)}(\eta) = 2\prod_{i=1}^k (1-\kappa\eta_k/2),
    \end{align*}
    and,
    \begin{align*}
        u_k^{(2)}(\eta) = \beta^2d\sum_{i=1}^k\left[\eta_i^2 \{\kappa^{-1} + \eta_i\} \left\{2 + \frac{\beta^2\eta_i}{\alpha} + \frac{\beta^2\eta_i^2}{6}\right\} \prod_{j=i+1}^k(1-\kappa \eta_j/2) \right].
    \end{align*}
    We denote by $z^*$, the unique minimum of the potential function $U$. 
\end{lemma}

\noindent We consider the following proposition to derive an asymptotic bound on Wasserstein distance, for a homogeneous ULA:
\begin{proposition}\label{lem:ula_conv}
    Under the assumptions \ref{ass:grad_lip} and \ref{ass:strong_conv}, assuming we start the homogeneous ULA from $\pi^0 = \delta_z$ for some $z\in \R^d$,
    we get that for any $k\ge 1$,
    \begin{align*}
        W_2(\pi^k,\pi) \le \epsilon,
    \end{align*}
    using $\eta = O(\epsilon^2/d)$ and $k = \Omega\left(\frac{d}{\epsilon^2}\log\left(\frac{d+\|z-z^*\|^2}{\epsilon^2}\right)\right)$, where we hide constants that depend on $\alpha,\beta$. Here $z^*$ is the unique minimizer of the potential function $U$.
\end{proposition}
\begin{proof}
    We simplify terms from Lemma \ref{lem:wass_ub}:
     \begin{align*}
        u_k^{(1)}(\eta) &= 2\prod_{i=1}^k (1-\kappa\eta_k/2)\\
        &= 2(1-\kappa\eta/2)^k\\
        &\le 2\exp(-\frac{k\kappa \eta}{2}).
    \end{align*}
    And,
    \begin{align*}
        u_k^{(2)}(\eta) &= \beta^2d\sum_{i=1}^k\left[\eta_i^2 \{\kappa^{-1} + \eta_i\} \left\{2 + \frac{\beta^2\eta_i}{\alpha} + \frac{\beta^2\eta_i^2}{6}\right\} \prod_{j=i+1}^k(1-\kappa \eta_j/2) \right]\\
        &= \beta^2d\sum_{i=1}^k\left[\eta^2 \{\kappa^{-1} + \eta\} \left\{2 + \frac{\beta^2\eta}{\alpha} + \frac{\beta^2\eta^2}{6}\right\} (1-\kappa \eta/2)^{(k-i)} \right]\\
        &= \beta^2d\eta^2 \{\kappa^{-1} + \eta\} \left\{2 + \frac{\beta^2\eta}{\alpha} + \frac{\beta^2\eta^2}{6}\right\}\frac{1-(1-\kappa\eta/2)^k}{\kappa\eta/2}\\
        &\le 2\frac{\beta^2d\eta}{\kappa}\{\kappa^{-1} + \eta\} \left\{2 + \frac{\beta^2\eta}{\alpha} + \frac{\beta^2\eta^2}{6}\right\}\\
        &= \frac{\beta^2d\eta}{\alpha^2}\frac{(\alpha+\beta)}{\beta}\left\{\frac{\alpha+\beta}{2\beta} + \alpha\eta\right\}\left\{2 + \frac{\beta^2\eta}{\alpha} + \frac{\beta^2\eta^2}{6}\right\}\\
        &\le 2\frac{\beta^2d\eta}{\alpha^2}\{1 + \alpha\eta\}\left\{2 + \frac{\beta^2\eta}{\alpha} + \frac{\beta^2\eta^2}{6}\right\}\quad 
        \text{(since $\alpha\le \beta$)}\\
        &\le const.\times d\eta (\beta^2/\alpha^2 + \beta^4/\alpha^4) \le const.\times d\eta \beta^4/\alpha^4.
    \end{align*}
    Where in the last step we set $\eta\le 1/(\alpha+\beta)\implies \eta\le 1/\alpha$. We want $u_k^{(2)}(\eta) < \epsilon^2/2$, which happens with $\eta \le const.\times \alpha^4\epsilon^2/(dL^4)$. Finally, we want $u_k^{(1)}(\eta)\{\|z-z^*\|^2 + d/\alpha\} \le \epsilon^2/2$, which we get with $k = \Omega\left(\frac{d}{\epsilon^2}\log\left(\frac{d+\|z-z^*\|^2}{\epsilon^2}\right)\right)$,
    where we hide constants that depend on $\alpha$ and $\beta$. 
\end{proof}

We will be analyzing MCMC using samples from ULA. Define $\|g\|_{Lip}$ to be the Lipschitz constant of $g$, and the empirical mean starting from the $B$th iterate of homogeneous ULA as follows:
\begin{equation}
    \hat{\pi}_b^B(g) = \frac{1}{b}\sum_{i=B+1}^{b+B} g(X_i).
\end{equation}
We also define the expected value of $g$ as, $\pi(g) = \E_\pi[g]$. The following result on convergence of empirical mean will be crucial to our analysis of MCMC.
\begin{theorem}[Thm 15, 17, \citep{durmus_ula_19}]\label{thm:prob_mcmc}
Under the assumptions \ref{appass:grad_lip} and \ref{appass:strong_conv}, for all $B\ge 0, b\ge 1, r>0$ and Lipschitz functions $g:\R^d\rightarrow\R$:
    \begin{align*}
        \var(\hat{\pi}_b^B(g)) \le \frac{8\|g\|_{Lip}^2\left(1+\frac{(\kappa^{-1} + 2/(\alpha+\beta))}{b\eta}\right)}{\kappa^2b\eta},
    \end{align*}
    and,
    \begin{align*}
        Pr(\hat{\pi}_b^B(g) - \E[\hat{\pi}_b^B(g)] \ge r) \le \exp\left(-\frac{r^2\kappa^2b\eta}{16\|g\|^2_{Lip} \left(1+\frac{(\kappa^{-1} + 2/(\alpha+\beta))}{b\eta}\right)}\right).
    \end{align*}
\end{theorem}

Finally, we are interested in the gap between $\E[\hat{\pi}_b^B(g)]$ and $\pi(g)$. We recall the following from Section 4 of \citet{durmus_ula_19},
\begin{align*}
    \{\E[\hat{\pi}_b^B(g)] - \pi(g)\}^2 &\le \frac{\|g\|^2_{Lip}}{b} \sum_{i=B+1}^{b+B} W_2^2(\pi^i, \pi)
\end{align*}

We combine these results into the following proposition, which bounds the number of iterations and step-size for constant error in empirical mean with high probability.
\begin{proposition}[Restated Proposition \ref{prop:bounds}]
    If we run a homogeneous ULA from $\pi^0 = \delta_z$ for some $z\in \R^d$ with potential function $U$ that satisfies the assumptions \ref{appass:grad_lip} and \ref{appass:strong_conv} and $\eta=O(\epsilon^2/(\|g\|_{Lip}^2d))$, $B = \Omega\left(\frac{d\|g\|_{Lip}^2}{\epsilon^2}\log\left(\frac{(d+\|z-z^*\|^2)\|g\|_{Lip}^2}{\epsilon^2}\right)\right)$
    and $b$ that satisfies $b\eta \ge \frac{128\|g\|_{Lip}^2\log(1/\delta)}{\epsilon^2\kappa^2}$,
    \begin{align*}
        \var(\hat{\pi}_b^B)&\le \frac{\epsilon^2}{8\log(1/\delta)} \quad \text{and,}\\
        |\hat{\pi}_b^B(g) - \pi(g)| &\le \epsilon \quad \text{w.p. $\ge 1-\delta$}.
    \end{align*}
\end{proposition}
\begin{proof}
From the definition of the Wasserstein distance, $W_2^2(\pi^i,\pi) = \inf_{\gamma\in\Pi(\pi^i,\pi)}\E_{(x,y)\sim\gamma}[\|x-y\|^2]$, for $\Pi$ being the set of all couplings. In the following we assume the optimal coupling $\gamma^i$ between $\pi^i$ and $\pi$ such that $W_2^2(\pi^i,\pi) = \E_{(x,y)\sim\gamma^i}[\|x-y\|^2]$, which allows us to upper bound $|\E[\hat{\pi}_b^B(g)]-\pi(g)|$ as follows:
\begin{align*}
    \{\E[\hat{\pi}_b^B(g)] - \pi(g)\}^2 &= \left\{\frac{1}{b} \sum_{k=B+1}^{B+b} \E_{\pi,\pi^k}[g(z_{(k)}) - g(z)]\right\}^2 \\
    &\le \left\{\frac{\|g\|_{Lip}}{b} \sum_{k=B+1}^{B+b} \E[z_{(k)} - z]\right\}^2\\
    &\le \frac{\|g\|_{Lip}^2}{b^2}\sum_{k=B+1}^{B+b} \E[z_{(k)} - z]^2 \\
    &\le \frac{\|g\|^2_{Lip}}{b} \sum_{i=B+1}^{B+b} W_2^2(\pi^i, \pi).
\end{align*}
We know from Proposition \ref{lem:ula_conv} that $W_2(\pi^i,\pi) \le \epsilon$ for all $i\ge B$ for $\eta = O(\epsilon^2/d)$ and $B = \Omega\left(\frac{d}{\epsilon^2}\log\left(\frac{d+\|z-z^*\|^2}{\epsilon^2}\right)\right)$.
Thus, $|\E[\hat{\pi}_b^B(g)] - \pi(g)| \le \epsilon/2$ if $\eta=O(\epsilon^2/(\|g\|_{Lip}^2d))$ and $B = \Omega\left(\frac{d\|g\|_{Lip}^2}{\epsilon^2}\log\left(\frac{(d+\|z-z^*\|^2)\|g\|_{Lip}^2}{\epsilon^2}\right)\right)$,
which completes the first part of the proof.

Next, we need to bound $\hat{\pi}_b^B(g) - \E[\hat{\pi}_b^B(g)]$. For the derived value of $\eta = O(\epsilon^2/(\|g\|_{Lip}^2d))$ and a $b$ that satisfies $b\eta \ge \frac{128\|g\|_{Lip}^2\log(1/\delta)}{\epsilon^2\kappa^2}$, we have that $b\eta\gg 1$ resulting in $\var(\hat{\pi}_b^B)\le \frac{\epsilon^2}{8\log(1/\delta)}$ using Theorem \ref{thm:prob_mcmc}. Theorem \ref{thm:prob_mcmc} would thus also imply $|\hat{\pi}_b^B(g) - \E[\hat{\pi}_b^B(g)]| \le \epsilon/2$ with probability $\ge 1-\delta$, which finishes the proof.

\end{proof}
\begin{remark}
    We use the above Proposition to prove Theorem \ref{thm:lgd_conv}, where we replace $g(.)$ with $g(X_v;.)$, such that $\hat{\pi}_b^B$ is equivalent to $\hat{g}_v$ and the distribution $\pi$ is equivalent to the posterior distribution $w|X,y$.    
\end{remark}


\section{Tighter Pseudo-Dimension bounds for Regularized Logistic Regression}\label{sec:rlr_pdim}
We present a brief discussion improving the bounds on learning hyperparameters for regularized logistic regression from prior work. We consider the setting of \citet{balcan_23}, where the authors consider logistic regression regularized with either the $l1$ or the $l2$ penalty. We first re-define the setting of \citet{balcan_23} using notation from this paper formally. 

We consider classification on $m=2$ classes using the linear model $f(x;w)=x^\intercal w$. Since we define classification models to have outputs defined in $\R^m$, we re-write the model as $f(x;w)=[-x^\intercal w/2,x^\intercal w/2]$. We define the loss function
\begin{align*}
    l_{RLR}(y,f(X;w)) &= \sum_i \left(-\log\left(\frac{\exp(f(X^{(i)};w)[y^{(i)}])}{\exp(f(X^{(i)};w)[0]) + \exp(f(X^{(i)};w)[1])}\right)\right)\\
    &= \sum_i \left(-\log\left(\frac{1}{1 + \exp((-1)^{y^{(i)}})X^{(i)\intercal} w}\right)\right)\\
    &= -\log\left(\Pi_i\left(\frac{1}{1 + \exp((-1)^{y^{(i)}})X^{(i)\intercal} w}\right)\right).
\end{align*}
Consider regularizers which satisfy either $R(w;\theta) = \theta\|w\|_1$ or $R(w;\theta) = \theta \|w\|_2^2$ for $\theta\in[\theta_{min},\theta_{max}]$ with $\theta_{min}>0$. This setting corresponds exactly to regularized logistic regression with either $l1$ or $l2$ penalty per the treatment of \citet{balcan_23}. \citet{balcan_23} note that there is no closed-form solution to the optimization problem: $\hat{w}(X,y;\theta) = \argmin_w l_{RLR}(y,f(X;w)) + R(w;\theta)$. They correspondingly consider an optimization procedure that produces solutions that are $\epsilon$-close to the true solutions. We omit the discussion of the algorithm here, and defer the reader to \citet{balcan_23} for more details. We provide the approximation result of the algorithm as an existence result below.
 
\begin{theorem}[\citep{rosset_04, balcan_23}]
    There exists an algorithm to approximate the solution to $\hat{w}(X,y;\theta) = \argmin_w l_{RLR}(y,f(X;w)) + R(w;\theta)$, for $R$ corresponding to the $l1$ or $l2$ penalty such that the following are satisfied for a small enough $\epsilon$:
    \begin{enumerate}
        \item $\|w^{(\epsilon)}(X,y;\theta)-\hat{w}(X,y;\theta)\|=O(\epsilon^2)$
        \item $w^{(\epsilon)}(X,y;\theta)$ is piecewise linear in $\theta$, with $(\theta_{max}-\theta_{min})/\epsilon$ pieces,
    \end{enumerate}
    where $w^{(\epsilon)}(X,y;\theta)$ is the $\epsilon$-approximate algorithm.
\end{theorem}

The above result establishes that the outputs of the algorithm are piecewise linear functions. In order to analyze the pseudo-dimension of the validation loss function class, we invoke the following definition and result from \citet{hpa_balcan_21}.

\begin{definition}[Oscillations of a Function]\label{def:oscillation}
    A function $h:\R\rightarrow\R$ has at most $B$ oscillations if for every $z\in\R$, the function $\rho \rightarrow \mathbb{I}\{h(\rho)\ge z\}$ is piecewise constant with at most $B$ discontinuities.
\end{definition}

\begin{lemma}[Lemma 3.8, \citep{hpa_balcan_21}]\label{lem:osc_single}
    For some domain $\mathcal{Z}'$, let $\mathcal{G}$ be a class of functions
    $\mathcal{G} = \{g(z,.):\mathcal{Z}'\rightarrow\R|z\in\R\}$. If the functions $g(.,z')$ for $z'\in\mathcal{Z}'$ have at most $\mathcal{O}$ oscillations, then $Pdim(\mathcal{G}) = O(\ln B)$.
\end{lemma}

We are now ready to present our bound.

\begin{theorem}
    Assume an algorithm to $\epsilon$-approximate the solution $\hat{w}(X,y;\theta) = \argmin_w l_{RLR}(y,f(X;w)) + R(w;\theta)$ such that the following are satisfied for a small enough $\epsilon$:
    \begin{enumerate}
        \item $\|w^{(\epsilon)}(X,y;\theta)-\hat{w}(X,y;\theta)\|=O(\epsilon^2)$
        \item $w^{(\epsilon)}(X,y;\theta)$ is piecewise linear in $\theta$, with $(\theta_{max}-\theta_{min})/\epsilon$ pieces,
    \end{enumerate}
    where $w^{(\epsilon)}(X,y;\theta)$ is the $\epsilon$-approximate algorithm. For a task $(X,y,X_v,y_v)\in\mathcal{P}$, define $l_{v,RLR}(\theta;(X,y,X_v,y_v)) = l_{RLR}(y_v;f(X_v;w^{(\epsilon)}(X,y;\theta)))$. Define the function class $\mcL_v = \{l_{RLR}(\theta,.):\mathcal{P}\rightarrow\R|\theta\in[\theta_{min},\theta_{max}]\}$ as the class of validation loss functions for the given hyperparameter values. The pseudo-dimension of $\mcL_v$ satisfies:
    \begin{align*}
        Pdim(\mcL_v) = O(\log(n_v/\epsilon)).
    \end{align*}    
\end{theorem}
\begin{proof}
    In order to compute the pseudo-dimension of $\mcL_v$, we want to find the maximum number of oscillations of $l_{RLR}(.,(X,y,X_v,y_v))$ for any task $(X,y,X_v,y_v)$. We know that $w^{(\epsilon)}(X,y;.)$ is a piecewise linear function with at most $(\theta_{max}-\theta_{min})/\epsilon$ pieces. For any particular validation example $X_v^{(i)} $, $X_v^{(i)\intercal} w^{(\epsilon)}(X,y;.)$ is then a piecewise linear function with at most $(\theta_{max}-\theta_{min})/\epsilon$ pieces. Over all $n_v$ validation examples, we have a total of at most $n_v(\theta_{max}-\theta_{min})/\epsilon$ regions, such that $X_v^{(i)\intercal} w^{(\epsilon)}(X,y;.)$ is a linear function for all $i\in[n_v]$ inside any of these regions. Thus, inside any of these $n_v(\theta_{max}-\theta_{min})/\epsilon$ regions, all the functions $1 + \exp(X_v^{(i)\intercal} w^{(\epsilon)}(X,y;.))$, as well as $1 + \exp(-X_v^{(i)\intercal} w^{(\epsilon)}(X,y;.))$ are convex. In fact, $\Pi_{i\in[n_v]}\left(1 + \exp((-1)^{y_v^{(i)}}X_v^{(i)\intercal} w^{(\epsilon)}(X,y;.))\right)$ is a product of convex functions, and is hence convex in any of the $n_v(\theta_{max}-\theta_{min})/\epsilon$ regions. Now,
    \begin{align*}
        l_{RLR}(\theta,(X,y,X_v,y_v)) &\ge  z\\
        \iff \Pi_{i\in[n_v]}\left(1 + \exp((-1)^{y_v^{(i)}}X_v^{(i)\intercal} w^{(\epsilon)}(X,y;\theta))\right) &\ge \exp(z),
    \end{align*}
    so that the number of oscillations of $l_{RLR}(.,(X,y,X_v,y_v))$ is equal to the number of oscillations of $\Pi_{i\in[n_v]}\left(1 + \exp((-1)^{y_v^{(i)}}X_v^{(i)\intercal} w^{(\epsilon)}(X,y;.))\right)$. Due to convexity, the latter can have at most $2$ oscillations in any of the regions, such that  $l_{RLR}(.,(X,y,X_v,y_v))$ has at most $2n_v(\theta_{max}-\theta_{min})/\epsilon$ oscillations overall. The result then follows directly from Lemma \ref{lem:osc_single}.
\end{proof}

\begin{remark}
    Note that we do not consider a loss function that is normalized by the number of samples, so that the range $[\theta_{min},\theta_{max}]$ is conventionally constant. The treatment of \citet{balcan_23} considers average loss functions in addition to assuming the range $[\theta_{min},\theta_{max}]$ is a constant. By their convention, the bounds in the above theorem would drop to $O(\log(1/\epsilon))$, but we don't report those as the assumption is impractical.
\end{remark}

As we note in Section \ref{sec:guarantees}, our main results differ significantly from those of \citet{balcan_23} since we consider hardware approximations of the $\exp(.)$ function while \citet{balcan_23} assume perfect implementations.
While we maintain our focus on bounds that follow hardware restrictions, we note that bounds analogous to ours can be derived assuming perfect implementation of the $\exp(.)$ function following the treatment of \citet{poly_approx_vitercik2020}, where the authors prove upper bounds on the Rademacher complexity of a class of algorithms that can be implemented approximately using piecewise polynomials.

Finally, we propose the following extension of Lemma \ref{lem:osc_single} for function classes with multiple identifying parameters which may be instrumental for future work in the area.

\begin{lemma}
    Consider the set $\mathcal{G}=\{g(z,.):\mathcal{Z}'\rightarrow\R:z\in\R^h\}$ of functions specified by $h$ real parameters. Suppose that for each $z'\in\mathcal{Z}'$, each $i\in[h]$ and each vector $z^{-i}\in \R^{n-1}$, the function $g(\{.,z^{-i}\},z'):\R\rightarrow\R$ has at most $\mathcal{O}$ oscillations as defined in Definition \ref{def:oscillation}.  Then, $Pdim(\mathcal{G}) = O(h\log(h\mathcal{O}))$.
\end{lemma}
\begin{proof}
    We proceed with this proof similar to the proof of Lemma 3.8 in \citet{hpa_balcan_21}.
    
    Let's assume that the Pseudo-dimension of $\mathcal{G}$ is $v$. Then, there are $v$ inputs $\{z'_1,\ldots, z'_v\}\subseteq\mathcal{Z}'$ and $v$ real numbers $r_1,\ldots, r_v$ such that for every $b\in\{0,1\}^v$, there is a real vector $z_b$ that satisfies $sign(g(z_b,x_i)-r_i)=b_i,\;\forall i\in [v]$. We know from the theorem statement that $g(\{.,z^{-j}\},z')$ has at most $\mathcal{O}$ oscillations for any $j\in[h]$. We can thus partition the real line into $\le v\mathcal{O}+1$ intervals such that each $sign(g(\{.,z^{-j}\},z')-r_i)$ is constant in each interval. Repeating the same procedure over all $h$ dimensions, we get $(v\mathcal{O}+1)^h$ connected components where the function $z\rightarrow sign(g(z,.)-r_i)$ is constant. We want this function to take on $2^v$ distinct values which gives us $2^v \le (v\mathcal{O}+1)^h$. Taking a log on both sides we get $v\le n\log v + n \log(B+1)\implies v = O(n\log(nB))$ using arithmetic results from \citet{ML_theory_Shwartz_14}.
\end{proof}

\end{document}